\documentclass{article}

\usepackage{microtype}
\usepackage{graphicx}
\usepackage{booktabs} 

\usepackage{hyperref}

\usepackage[accepted]{icml}

\usepackage{amsmath,amssymb,amsfonts}
\usepackage{pifont}
\usepackage{mathtools}
\usepackage{bbm}
\usepackage{cases}
\usepackage{subcaption}
\usepackage{booktabs}
\usepackage{multirow}
\usepackage{overpic}
\usepackage{wrapfig}
\usepackage{xifthen}
\usepackage{multicol}
\usepackage{cuted}
\usepackage{enumitem}
\usepackage[switch]{lineno}
\usepackage{hyperref,soul}

\newcommand{\neural}[1][]{
\ifthenelse{\isempty{#1}}
{\mathbf{f}}
{\mathbf{f}^{(#1)}}}
\newcommand{\Jac}[1][]{
\ifthenelse{\isempty{#1}}
{\mathbf{J}}
{\mathbf{J}^{(#1)}}}
\newcommand{\J}[1][]{
\ifthenelse{\isempty{#1}}
{\boldsymbol{J}}
{\boldsymbol{J}^{(#1)}}}
\newcommand{\hJ}[1][]{
\ifthenelse{\isempty{#1}}
{\widehat{\boldsymbol{J}}}
{\boldsymbol{J}'^{(#1)}}}
\newcommand{\kernel}[1][]{\mathbf{H}^{(#1)}}
\newcommand{\kernelinf}[1][k]{\mathbf{H}_{\infty}^{(#1)}}
\newcommand{\pkernel}[1][]{\mathbf{H}^{\perp(#1)}}
\newcommand{\asyker}[1][]{\mathbf{\Lambda}^{(#1)}}
\newcommand{\pasyker}[1][]{\mathbf{\Lambda}^{\perp(#1)}}
\newcommand{\labels}{\mathbf{Y}}

\newcommand{\dset}{\mathcal{D}}

\newcommand*{\diff}{\mathop{}\!\mathrm{d}}

\newcommand{\x}{\mathbf{x}}
\newcommand{\y}{\mathbf{y}}

\newcommand{\X}{\mathbf{X}}
\newcommand{\Y}{\mathbf{Y}}
\newcommand{\Z}{\mathbf{Z}}

\newcommand{\hy}{\hat{\mathbf{y}}}
\newcommand{\veca}{\mathbf{a}}

\newcommand{\vecd}{\mathbf{d}}

\newcommand{\V}[1][]{
\ifthenelse{\isempty{#1}}
{\mathbf{V}}
{\mathbf{V}^{(#1)}}}
\newcommand{\vecv}[1][]{
\ifthenelse{\isempty{#1}}
{\mathbf{v}}
{\mathbf{v}^{(#1)}}}

\newcommand{\prob}{\mathbb{P}}

\newcommand{\qw}[1][]{\mathbf{w}}

\newcommand{\epsidx}[1][]{
\ifthenelse{\isempty{#1}}
{\varepsilon^{(t)}}
{\varepsilon^{(t)}_{#1}}}

\newcommand{\res}[1][]{  
\ifthenelse{\isempty{#1}}
{\mathbf{R}}
{\mathbf{R}^{(#1)}}}

\newcommand{\w}[1][]{
\ifthenelse{\isempty{#1}}
{\mathbf{w}}
{\mathbf{w}^{(#1)}}}
\newcommand{\tw}[1][]{
\ifthenelse{\isempty{#1}}
{\widetilde{\mathbf{w}}}
{\widetilde{\mathbf{w}}^{(#1)}}}

\newcommand{\argmin}{\mathop{\mathrm{argmin}}\limits} 

\newcommand{\cset}{\mathcal{C}}

\newcommand{\erf}{\textrm{erf}}

\newcommand{\ind}[1][]{
\ifthenelse{\isempty{#1}}
{\mathbbm{1}}
{\mathbbm{1} \left(#1\right) }}
\newcommand{\indi}[1][]{
\ifthenelse{\isempty{#1}}
{\mathbbm{1}}
{\mathbbm{1}^{(#1)}}}

\newcommand{\expect}[1][]{
\ifthenelse{\isempty{#1}}
{\mathbb{E}}
{\mathbb{E}\left[#1\right]}}

\newcommand{\finnerprod}[2]
{\left\langle #1,\,#2 \right\rangle_{\textrm{F}}}

\newcommand{\I}{\mathcal{I}}

\newcommand{\normsq}[1]
{\big\| #1\big\|_2^2}

\newcommand{\order}[1]{O\left( #1 \right)}

\newcommand{\RN}[1]{%
  \textup{\uppercase\expandafter{\romannumeral#1}}%
}

\newcommand{\pdist}[1][m]{\mathcal{P}_{m}}

\newcommand{\p}[1][]{
\ifthenelse{\isempty{#1}}
{\boldsymbol{p}}
{\boldsymbol{p}^{(#1)}}
}

\newcommand{\vecr}[1][]{
\ifthenelse{\isempty{#1}}
{\mathbf{r}}
{\mathbf{r}^{(#1)}}}

\newcommand{\sset}{\mathcal{S}}

\newcommand{\cirone}
{\text{\ding{172}}}
\newcommand{\cirtwo}
{\text{\ding{173}}}
\newcommand{\cirthree}
{\text{\ding{174}}}
\newcommand{\cirfour}
{\text{\ding{175}}}
\newcommand{\cirfive}
{\text{\ding{176}}}

\newcommand{\intab}[2][0.75]{
\scalebox{#1}{\textrm{#2}}
}

\newenvironment{proof}[1][]{
\ifthenelse{\isempty{#1}}
{\par\vspace*{-1mm}\noindent\textit{Proof.} }
{\par\vspace*{-2mm}\noindent\textit{Proof of #1.} }}
{\hfill$\square$}

\newtheorem{Theorem}{Theorem}
\newtheorem{Lemma}{Lemma}

\newtheorem{Remark}{Remark}
\newtheorem{Assumption}{Assumption}

\usepackage{tikz}

\newcommand{\highlight}[1]{\vspace{1mm}\noindent{}\textbf{#1}}
\newcommand{\papertheorem}[1]{Theorem~\ref{#1}}
\newcommand{\paperassumption}[1]{Assumption~\ref{#1}}
\newcommand{\paperlemma}[1]{Lemma~\ref{#1}}

\newcommand{\paperfig}[1]{Figure~\ref{#1}}
\newcommand{\papertab}[1]{Table~\ref{#1}}

\newcommand{\paperappendix}[1]{Appendix~\ref{#1}}

\newcommand{\zoom}[2][0.75]{  
\scalebox{#1}{#2}
}

\newcommand\doublerule{\toprule\specialrule{\heavyrulewidth}{\doublerulesep}{0.95em}}

\ifodd 0  
\newcommand{\revKY}[1]{\textcolor[rgb]{0.04,0.5,1.00}{#1}}
\else
\newcommand{\revKY}[1]{#1}
\fi

\icmltitlerunning{Neural Tangent Kernel Empowered Federated Learning}

\begin{document}

\twocolumn[
\icmltitle{Neural Tangent Kernel Empowered Federated Learning}




\begin{icmlauthorlist}
\icmlauthor{Kai Yue}{ncsu}
\icmlauthor{Richeng Jin}{ncsu}
\icmlauthor{Ryan Pilgrim}{scholar}
\icmlauthor{Chau-Wai Wong}{ncsu}
\icmlauthor{Dror Baron}{ncsu}
\icmlauthor{Huaiyu Dai}{ncsu}
\end{icmlauthorlist}

\icmlaffiliation{ncsu}{NC State University}
\icmlaffiliation{scholar}{Independent Scholar}

\icmlcorrespondingauthor{Chau-Wai Wong}{chauwai.wong@ncsu.edu}

\icmlkeywords{Machine Learning, ICML}
\vskip 0.3in
]



\printAffiliationsAndNotice{}  

\begin{abstract}
Federated learning (FL) is a privacy-preserving paradigm where multiple participants jointly solve a machine learning problem without sharing raw data. 
Unlike traditional distributed learning, a unique characteristic of FL is statistical heterogeneity, namely, data distributions across participants are different from each other. 
Meanwhile, recent advances in the interpretation of neural networks have seen a wide use of neural tangent kernels (NTKs) for convergence analyses.
In this paper, we propose a novel FL paradigm empowered by the NTK framework.
The paradigm addresses the challenge of statistical heterogeneity by transmitting update data that are more expressive than those of the conventional FL paradigms.
Specifically, sample-wise Jacobian matrices, rather than model weights/gradients, are uploaded by participants.
The server then constructs an empirical kernel matrix to update a global model without explicitly performing gradient descent.
We further develop a variant with improved communication efficiency and enhanced privacy.    
Numerical results show that the proposed paradigm can achieve the same accuracy while reducing the number of communication rounds by an order of magnitude compared to federated averaging. 
\end{abstract}

\section{Introduction}
Federated learning (FL) has emerged as a popular paradigm involving a large number of clients collaboratively solving a machine learning problem~\citep{kairouz2021advances}.  
In a typical FL framework, a server broadcasts a global model to selected clients and collects model updates without accessing the raw data.  
One popular algorithm is known as \emph{federated averaging} (FedAvg)~\citep{mcmahan2017communication}, in which clients perform stochastic gradient descent (SGD) to update the local models and then upload the weight vectors to the server.
A new global model is constructed on the server by averaging the received weight vectors.

\cite{li2020flsurvey} characterized some unique challenges for FL. 
First, client data are generated locally and remain decentralized, which implies that they may not be independent and identically distributed~(IID). 
Prior work has shown that statistical heterogeneity can negatively affect the convergence of FedAvg~\citep{zhao2018federated}. 
This phenomenon may be explained by noting that local updating under data heterogeneity will cause cost-function inconsistency~\citep{wang2020tackling}.
A more challenging issue is the system heterogeneity, including the diversity of client hardware, battery power, and network connectivity. 
Local updating schemes often exacerbate the straggler issue caused by heterogeneous system characteristics~\citep{li2020flsurvey}.

Recent studies have proposed various strategies to alleviate the statistical heterogeneity. 
One possible solution is to share a globally available dataset with participants to reduce the distance between client-data distributions and the population distribution~\citep{zhao2018federated}.
In practice, though, such a dataset may be unavailable or too small to meaningfully compensate for the heterogeneity.  
Some researchers replaced the coordinate-wise weight averaging strategy in FedAvg with nonlinear aggregation schemes~\citep{ wang2020federated, chen2021fedbe}.
The nonlinear aggregation relies on a separate optimization routine, which can be elusive, especially when the federated model does not perform well.  
Another direction is to modify the local objectives or local update schemes to cancel the effects of client drift~\citep{li2020federated, karimireddy2020scaffold}. 
However, some studies reported that these methods are not consistently effective when evaluated in various settings~\citep{reddi2021adaptive, haddadpour2021federated, chen2021fedbe}. 

In this work, we present a neural tangent kernel empowered federated learning (NTK-FL) paradigm.  
NTK-FL outperforms state-of-the-art methods by achieving the target accuracy with fewer communication rounds.   
We summarize our contributions as follows.
\begin{itemize}[leftmargin=*]
    \item We propose a novel FL paradigm without requiring clients to perform local gradient descent. 
    To the best of our knowledge, this is the first work using the NTK method to replace gradient descent to diversify the design of FL algorithms. 
    \item Our scheme inherently solves the non-IID data problem of FL. 
    Compared to FedAvg, it is robust to different degrees of data heterogeneity and has a consistently fast convergence speed.       
    We verify the effectiveness of the paradigm theoretically and experimentally. 
    \item We add \textbf{c}ommunication-efficient and \textbf{p}rivacy-preserving features to the paradigm and develop CP-NTK-FL by combining strategies such as random projection and data subsampling. 
    We show that some strategies can also be applied to traditional FL methods. 
    Although such methods cause performance degradation when applied to FedAvg, they only slightly worsen the model accuracy when applied to the proposed CP-NTK-FL.
\end{itemize}

\section{Related Work}

\highlight{Neural Tangent Kernel.}
\cite{jacot2018neural} showed that training an infinitely wide neural network with gradient descent in the parameter space is equivalent to kernel regression in the function space.
\cite{lee2019wide} used a first-order Taylor expansion to approximate the neural network output and derived the training dynamics in a closed form.
For the analyses, \cite{chen2020a} established the generalization bounds for a two-layer over-parameterized neural network with the NTK framework. 
The NTK computation has been extended to convolutional neural networks (CNNs)~\citep{arora2019on}, recurrent neural networks (RNNs)~\citep{alemohammad2021the}, and even to neural networks with arbitrary architectures~\citep{yang2021tensor}.
Empirical studies have also provided a good understanding of the wide neural network training~\citep{lee2020finite}. 

\highlight{Federated Learning.}
FL aims to train a model with distributed clients without transmitting local data~\citep{mcmahan2017communication,kairouz2021advances}. 
FedAvg has been proposed as a generic solution with theoretical analyses and implementation variants.  
Recent studies have shown a growing interest in improving its communication efficiency, privacy guarantees, and robustness to heterogeneity.
To reduce communication cost, gradient quantization and sparsification were incorporated into FedAvg~\citep{reisizadeh2020fedpaq, sattler2019robust}.
From the security perspective, \cite{zhu2019deep} showed that sharing gradients may cause privacy leakage in the model inversion attack. 
To address this challenge, differentially private federated optimization and decentralized aggregation methods were developed~\citep{girgis2021shuffled,cheng2021separation}.
Other works put the focus on the statistical heterogeneity issue and designed methods such as adding regularization terms to the objective function~\citep{li2020federated, smith2017federated} or employing personalized models~\citep{liang2019think}. 
In this work, we focus on a novel FL paradigm where the global model is derived based on the NTK evolution. 
We show that the proposed NTK-FL is robust to statistical heterogeneity inherently, and extend it to a variant with improved communication efficiency and enhanced privacy. 

\highlight{Kernel Methods in Federated Learning.}
The NTK framework has been mostly used for convergence analyses in FL.
\cite{seo2020federated} studied two knowledge distillation methods in FL and compared their convergence properties based on the neural network function evolution in the NTK regime. 
\cite{li2021fedbn} incorporated batch normalization layers to local models and provided theoretical justification for its faster convergence by studying the minimum nonnegative eigenvalue of the tangent kernel matrix. 
To facilitate the understanding of the conventional FL process, \cite{huang2021fl} directly used the NTK framework to analyze the convergence rate and generalization bound of two-layer ReLU neural networks trained with FedAvg.
Likewise, \cite{su2021achieving} studied the convergence behavior of a set of FL algorithms in the kernel regression setting.
In comparison, our work does not focus on pure convergence analyses of existing algorithms.
We propose a novel FL framework by replacing the gradient descent with the NTK evolution.  
\section{Background and Preliminaries}
Symbol conventions are as follows.
We use $[N]$ to denote the set of the integers $\{1, 2,\dots,N\}$. 
Lowercase nonitalic boldface, nonitalic boldface capital, and italic boldface capital letters denote column vectors, matrices, and tensors, respectively.
For example, for column vectors $\mathbf{a}_{j}\in \mathbb{R}^{M}, \, j \in [N]$, $\mathbf{A} = [\mathbf{a}_{1}, \ldots, \mathbf{a}_{N}]$ is an $M\times N$ matrix.
A third-order tensor $\boldsymbol{A} \in \mathbb{R}^{K\times M \times N}$ can be viewed as a concatenation of such matrices.
We use a \textit{slice} to denote a matrix in a third-order tensor by varying two indices \citep{kolda2009tensor}.
Take tensor $\boldsymbol{A}$, for instance: $\boldsymbol{A}_{i::}$ is a matrix of the $i$th horizontal slice, and $\boldsymbol{A}_{:j:}$ is its $j$th lateral slice \citep{kolda2009tensor}. 
The indicator function of an event is denoted by $\ind[\cdot]$.

\subsection{Federated Learning Model}
Consider an FL architecture where a server trains a global model by indirectly using datasets distributed among $M$ workers. 
The local dataset of the $m$th worker is denoted by $\mathcal{D}_{m} = \{(\x_{m,i}, \y_{m,i})\}_{i=1}^{N_m}$, 
where $(\x_{m,i}, \y_{m,i})$ is an input-output pair.
The local objective can be formulated as an empirical risk minimization over $N_m$ training examples: 
\begin{equation}
    F_{m}(\w)  = \frac{1}{N_m} \sum_{i=1}^{N_m} R(\w; \x_{m,j}, \y_{m,i}),
\end{equation}
where $R$ is a sample-wise risk function quantifying the error of model with a weight vector $\w \in \mathbb{R}^{d}$ estimating the label $\y_{m,i}$ for an input $\x_{m,i}$. 
The global objective function is denoted by $F(\w)$, and the optimization problem may be formulated as:
\begin{equation}
    \min_{\w \in \mathbb{R}^{d}} F(\w) = \frac{1}{M} \sum_{m=1}^{M} F_{m}(\w).
\end{equation}

\subsection{Linearized Neural Network Model}\label{section:linearized_nn}
Let $(\x_i, \y_i)$ denote a training pair, with $\x_i \in \mathbb{R}^{d_1}$ and $\y_i \in \mathbb{R}^{d_2}$, where $d_1$ is the input dimension and $d_2$ is the output dimension.  
$\X \triangleq [\x_1, \dots, \x_N]^\top$ represents the input matrix and $\Y \triangleq [\y_1, \dots, \y_N]^\top$ represents the label matrix. 
Consider a neural network function $\neural: \mathbb{R}^{d_1} \rightarrow \mathbb{R}^{d_2}$ parameterized by a vector $\w \in \mathbb{R}^d$, which is the vectorization of all weights for the multilayer network. 
Given an input $\x_i$, the network outputs a prediction $\hy_i = \neural(\w;\x_i)$. 
Let $\ell(\hy_i, \y_i)$ be the loss function measuring the dissimilarity between the predicted result $\hy_i$ and the true label $\y_i$. 
We are interested in finding an optimal weight vector $\w^\star$ that minimizes the empirical loss over $N$ training examples:
\begin{equation}
    \w^\star = \argmin_{\w} L(\w; \X, \Y) 
    \triangleq \frac{1}{N} \sum_{i=1}^N \ell(\hy_i, \y_i).
\end{equation}

One common optimization method is the gradient descent training.
Given the learning rate $\eta$, gradient descent updates the weights at each time step as:
$
    \w[t+1] = \w[t] - \eta \nabla_{\w} L. 
$
To simplify the notation, let $\neural[t](\x) $ be the output at time step $t$ with an input $\x$, i.e., $\neural[t](\x) \triangleq \neural(\w[t];\x)$. 
Following \cite{lee2019wide}, we use the first-order Taylor expansion around the initial weight vector $\w[0]$ to approximate the neural network output given an input $\x$, i.e., 
\begin{equation}
    \neural[t](\x) \approx \neural[0](\x) + \Jac[0](\x)(\w[t] - \w[0]), 
\end{equation}
where $\Jac[0](\x) = [\nabla \neural[0]_1 (\x), \dots, \nabla \neural[0]_{d_2} (\x)]^\top$, with 
$\nabla \neural[t]_{j}(\x) \triangleq [\partial \hat{y}^{(t)}_j/\partial w^{(t)}_{1}, \ldots,  \partial \hat{y}^{(t)}_j/\partial w^{(t)}_{d}]^\top$ being the gradient of the $j$th component of the neural network output with respect to $\w[t]$.
Consider the halved mean-squared error (MSE) loss $\ell$, namely, 
$
    \ell(\mathbf{a}, \mathbf{b}) = \frac{1}{d_2} \sum_{j=1}^{d_2} \frac{1}{2} (a_{j} - b_{j})^2.
$
Based on the continuous-time limit, one can show that the dynamics of the gradient flow are governed by the following differential equation: 
\begin{equation}\label{eq:ode}
    \frac{\diff \neural}{\diff t} = -\eta \, \kernel[0]\big( \neural[t](\X) - \Y \big),
\end{equation}
where $\neural[t](\X) \in \mathbb{R}^{N \times d_2}$ is a matrix of concatenated output for all training examples,  
and $\kernel[0]$ is the neural tangent kernel at time step $0$, 
with each entry $(\kernel[0])_{ij}$ equal to the scaled Frobenius inner product of the Jacobian matrices: 
\begin{equation}
    (\kernel[0])_{ij} = \frac{1}{d_2} \finnerprod{\Jac[0](\x_{i}) }{\Jac[0](\x_{j})}. \label{eq:kernel_entry}
\end{equation}
The differential equation \eqref{eq:ode} has the closed-form solution:  
\begin{equation}\label{eq:ode_solution}
    \neural[t](\X) =  \left( \mathbf{I} -e^{- \frac{\eta t}{N} \kernel[0] }\right) \Y + e^{-\frac{\eta t}{N} \kernel[0] } \neural[0](\X). 
\end{equation}
The neural network state $\neural[t](\X)$ can thus be directly obtained from \eqref{eq:ode_solution} without running the gradient descent algorithm. 
\section{Proposed FL Paradigm via the NTK Framework}

In this section, we present the NTK-FL paradigm (\paperfig{subfig:ntk-fl}) and then extend it to the variant CP-NTK-FL (\paperfig{subfig:cp-ntk-fl}) with improved communication efficiency and enhanced privacy.  
The detailed algorithm descriptions are presented as follows. 
    
\subsection{NTK-FL Paradigm}
NTK-FL follows four steps to update the global model in each round.
\textbf{First}, the server will select a subset $\cset_k$ of clients and broadcast to them a model weight vector $\w[k]$ from the $k$th round. 
Here, the superscript $k$ is the communication round index, and it should be distinguished from the gradient descent time step $t$ in Section \ref{section:linearized_nn}.  
\textbf{Second}, each client will use its local training data $\dset_m$ to compute a Jacobian tensor $\J[k]_m \in \mathbb{R}^{N_m \times d_2 \times d}, \, \forall\; m \in \cset_{k}$, which is a concatenation of $N_m$ sample-wise Jacobian matrices
$(\J[k]_m)_{i::} = [\nabla \neural[k]_1 (\x_{m,i}), \dots, \nabla \neural[k]_{d_2} (\x_{m,i})]^\top$, $i \in [N_m]$.
The client will then upload the Jacobian tensor $\J[k]_m$, labels $\Y_m$, and initial condition $\neural[k](\X_m)$ to the server. 
The transmitted information corresponds to the variables in the solution for the state evolution in \eqref{eq:ode_solution}.
\textbf{Third}, the server will construct a global Jacobian tensor $\J[k] \in \mathbb{R}^{N_k \times d_2 \times d}$ based on received $\J[k]_m$'s, 
with each client contributing $N_{m}$ horizontal slices to $\J[k]$.
Here, we use $N_k$ to denote the number of training examples in communication round $k$ when there is no ambiguity.  

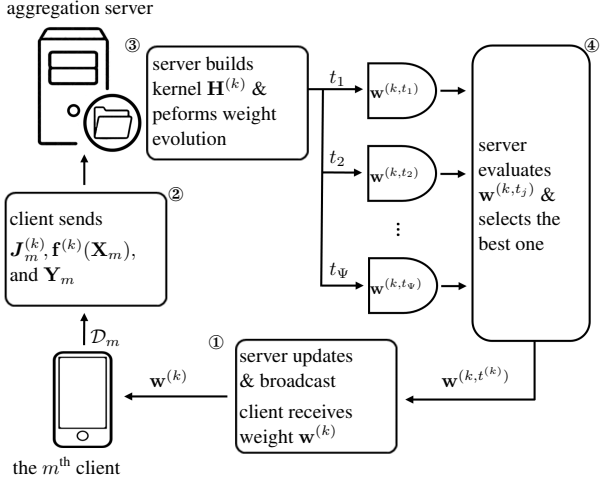
\begin{figure}
    \centering
    \vspace*{-4pt}
    \begin{overpic}[width=\linewidth]{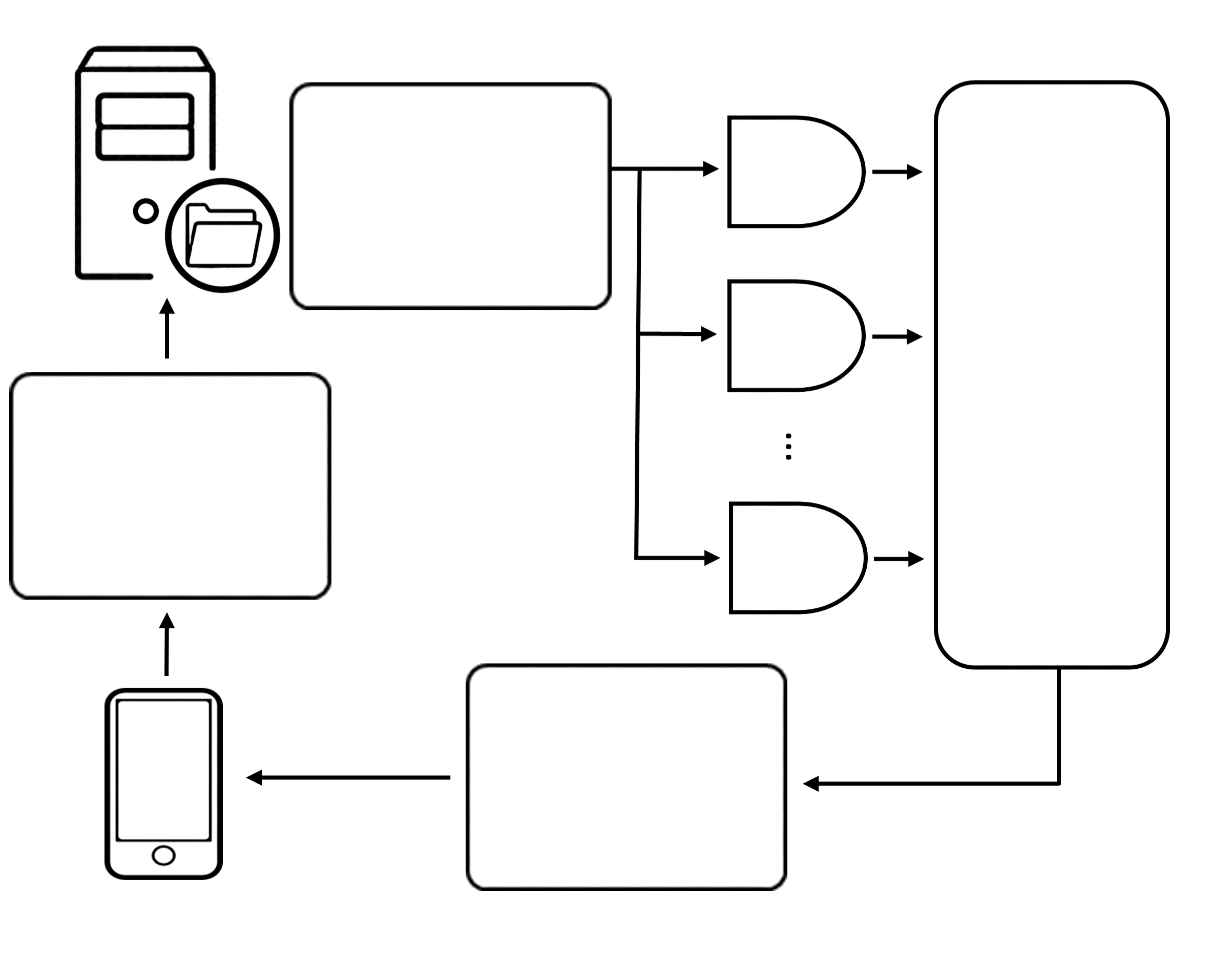}
        \put(1, 2){\zoom{the $m^{\text{th}}$ client}}

        \put(0, 76){\zoom{aggregation server}}

        \put(32.5, 22){\zoom{\cirone}}
        \put(37.5, 20){\zoom{server updates}}
        \put(37.5, 16){\zoom{\& broadcast}}  
        \put(37.5, 11){\zoom{client receives}}
        \put(37.5, 7){\zoom{weight $\w[k]$}}

        \put(23, 16){\zoom{$\w[k]$}}
        \put(13.5, 23){\zoom{$\dset_{m}$}}

        \put(26, 46){\zoom{\cirtwo}}
        \put(0.5, 42){\zoom{client sends}}
        \put(0.5, 37){\zoom{$\J[k]_{m}\!, \neural[k](\X_m),$}}
        \put(0.5, 33){\zoom{and $\Y_{m}$}}

        \put(19, 70){\zoom{\cirthree}} 
        \put(23.5, 67){\zoom{server builds}}
        \put(23.5, 63){\zoom{kernel $\kernel[k]$ \& }}
        \put(23.5, 59){\zoom{peforms weight }} 
        \put(23.5, 55){\zoom{evolution }}

        \put(58.5, 62.5){\zoom[0.65]{$\w[k, t_1]$}}
        \put(58.5, 49){\zoom[0.65]{$\w[k, t_2]$}}
        \put(58.5, 31){\zoom[0.65]{$\w[k, t_\Psi]$}}
        \put(52, 65){\zoom{$t_1$}}
        \put(52, 52){\zoom{$t_2$}}
        \put(52, 34){\zoom{$t_\Psi$}}

        \put(93, 70){\zoom{\cirfour}}
        \put(76, 54){\zoom{server}}
        \put(76, 50){\zoom{evaluates}}
        \put(76, 46){\zoom{$\w[k, t_j]$ \&}}
        \put(76, 42){\zoom{selects the}}
        \put(76, 38){\zoom{best one}}

        \put(70, 16){\zoom{$\w[k, t^{(k)}]$}}
    \end{overpic}\\[-5pt]
    \caption{
    Schematic of NTK-FL. 
    Each client first receives the weight $\w[k]$, 
    and then uploads the Jacobian tensor $\J[k]_m$, label $\Y_m$, and initial condition $\neural[k](\X_m)$.
    The server builds a global kernel $\kernel[k]$ and performs the weight evolution with $\{t_1, \ldots, t_\Psi\}$.
    We use \eqref{eq:t_optim} to find the best $t_j$ and update the weight accordingly. 
    \label{subfig:ntk-fl}
    }
\end{figure}

We use a toy example to explain the process as follows. 
Suppose the server selects client 1 and client 2 in a certain round. 
Clients 1 and 2 will compute the Jacobian tensors $\J[k]_1$ and $\J[k]_2$, respectively.
The global Jacobian tensor is constructed as:  
\begin{equation}
    \J[k]_{i::} = \left\{
        \begin{array}{l @{\; \;} l}
        \J[k]_{1, i::} & , \text { if } i \leqslant N_1, \\[3pt]
        \J[k]_{2, j::} & , j=i-N_1, \text { if } i \geqslant N_1 + 1. 
    \end{array}\right.
\end{equation}
After obtaining the global Jacobian tensor $\J[k]$, the $(i,j)$th entry of the global kernel $\kernel[k]$ is calculated as the scaled Frobenius inner product of two horizontal slices of $\J[k]$, i.e.,
$
    (\kernel[k])_{ij} = \frac{1}{d_2}  \langle\J[k]_{i::}, \J[k]_{j::}\rangle_{\text{F}} .  
$
\textbf{Fourth}, the server will perform the NTK evolution to obtain the updated neural network function $\neural[k+1]$ and weight vector $\w[k+1]$. 
With a slight abuse of notation, let $\neural[k,t]$ denote the neural network output at gradient descent step $t$ in communication round $k$. 
The neural network function evolution dynamics and weight evolution dynamics are given by:
\begin{subequations}
\begin{align}
    \neural[k, t] & = \left( \mathbf{I} - e^{- \frac{\eta t}{N_k} \kernel[k] }\right) \labels^{(k)} + e^{- \frac{\eta t}{N_k} \kernel[k] } \neural[k], \label{eq:f_evolution} \\
    \w[k, t] & =  \sum_{j=1}^{d_{2}}  (\J[k]_{:j:})^\top \res[k,t]_{:j} + \w[k], \label{eq:w_evolution}
\end{align}
\end{subequations}
where $\J[k]_{:j:}$ is the $j$th lateral slice of $\J[k]$, 
and $\res[k,t]_{:j}$ is the $j$th column of the residual matrix $\res[k,t]$ defined as follows:
\begin{equation}
    \res[k,t] \triangleq  \frac{\eta}{N_k d_{2}} \sum_{u=0}^{t-1} \big[\Y^{(k)} - \neural[k,u](\X^{(k)}) \big]. 
\end{equation}
Here, $\X^{(k)}$ and $\Y^{(k)}$ denote a concatenation of client training examples and labels, respectively.
The weight evolution in \eqref{eq:w_evolution} is derived by unfolding the gradient descent steps. 
To update the global weight, the server performs the evolution with various integer steps $\{t_1, \ldots, t_\Psi\}$ and selects the best one with the smallest loss value. 
Our goal is to minimize the training loss with a small generalization gap \citep{nakkiran2020deep}. 
The updated weight is decided by the following procedure: 
\begin{subequations}
\begin{align}
    t^{(k)} &= \argmin_{t_j} L( \neural(\w[k,t_j]; \X^{(k)}); \Y^{(k)}), \label{eq:t_optim} \\
    \w[k+1] &\triangleq \w[k,t^{(k)}].
\end{align}
\end{subequations}
Alternatively, if the server has an available validation dataset, the optimal number of update steps can be selected based on the model validation performance.  
In practice, such a validation dataset can be obtained from held-out clients \citep{wang2021field}. 
Based on the closed-form solution in \eqref{eq:w_evolution}, the  search of $t^{(k)}$ over the grid $\{t_1, \dots, t_\Psi\}$ can be completed in $O(\Psi)$ time.

\highlight{Robustness Against Statistical Heterogeneity. }
In essence, statistical heterogeneity comes from decentralized data with heterogeneous distributions owned by individual clients.
If privacy is not an issue, the non-IID challenge can be readily resolved by mixing all clients' datasets and training a centralized model.
In NTK-FL, instead of building a centralized dataset, we use Jacobian matrices to construct a global kernel $\kernel[k]$, 
which is a concise representation of gathered data points from all selected clients. 
This representation is yet more \emph{expressive}/less \emph{compact} than that of a traditional FL algorithm.
More precisely, the update being sent for NTK-FL regarding the $i$th training example of the $m$th client for NTK-FL is $ \Jac_m = [\nabla \mathbf{f}_{1}(\mathbf{x}_{m,i}), \dots, \nabla \mathbf{f}_{d_2}(\mathbf{x}_{m, i})]^\top$, 
whereas the gradient update being sent for FedAvg is $\nabla L(\w; \x_{m,i}, \y_{m,i}) =  \frac{1}{d_2}\sum_{j=1}^{d_{2}}\left(\hat{y}_{m, i, j}-y_{m, i, j}\right) \nabla \mathbf{f}_{j}(\mathbf{x}_{m, i})$, a weighted sum of row vectors in $\Jac_m$. 
The gradient will be further averaged over multiple training examples.  
By sending Jacobian matrices $\Jac_m$ and jointly processing them on the server, NTK-FL delays the more aggressive data aggregation step after the communication stage and therefore better approximates the centralized learning setting than FedAvg does.

\highlight{Comparison of NTK-FL and \cite{huang2021fl}. }
\cite{huang2021fl} presented the details of FedAvg by letting clients use local updates and upload gradients to train a two-layer neural network.
In contrast, NTK-FL let each client transmit Jacobian matrices without performing local SGD steps. 
The model weight is updated via NTK evolution in \eqref{eq:w_evolution}. 
The main differences include: (1) clients transmit more expressive Jacobian matrices to improve model performance in the non-IID FL setting;
(2) more computation is shifted to the server.

\begin{figure*}
    \centering
    \begin{minipage}[b]{0.5\textwidth}
        \centering
        \begin{overpic}[width=\linewidth]{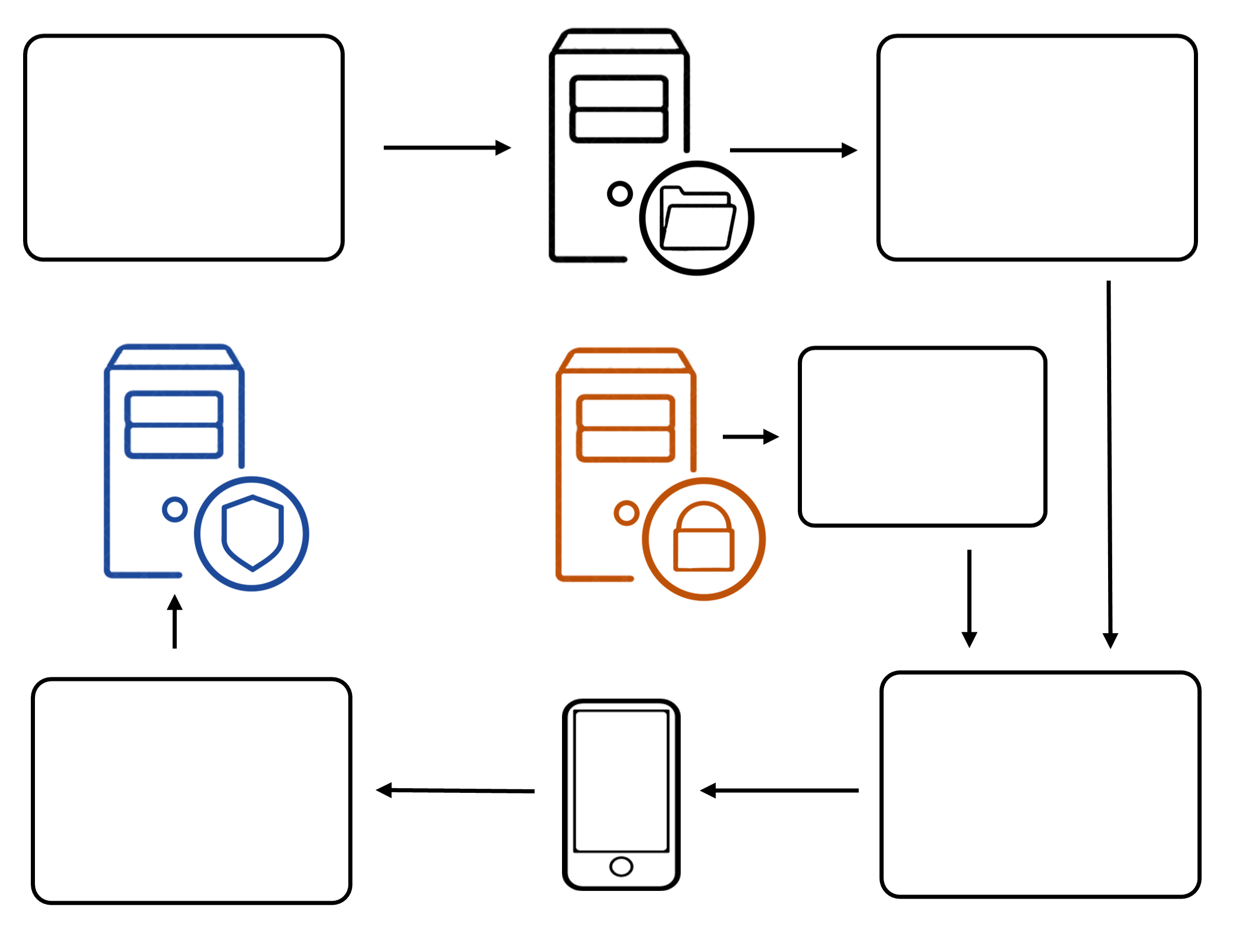}
            \put(1, 51){\zoom{\textbf{shuffling server}}}
            \put(36, 77){\zoom{aggregation server }}
            \put(36, 51){\zoom{\textbf{trusted key server}}}
            \put(40, 1){\zoom{the $m^{\text{th}}$ client}}
    
            \put(83.5, 49){\zoom{\cirone}}
            \put(64.6, 45){\zoom{key server}}
            \put(64.6, 41){\zoom{encrypts $\rho$}}
            \put(64.6, 37){\zoom{\& transmits}}
            \put(63, 28){\zoom{$\text{E}(\text{k}^+_m, \rho)$}}
    
            \put(66.5, 21){\zoom{\cirtwo}}
            \put(71, 18){\zoom{client receives}}
            \put(71, 14){\zoom{weight $\w[k]$ \&}}
            \put(71, 10){\zoom{decrypts to get}}
            \put(71, 6){\zoom{the seed $\rho$}}
            \put(57, 15){\zoom{$\w[k]\!, \rho$}}
            \put(28.5, 15){\zoom{$\mathcal{B}_m \!\subset\! \dset_m$}}
    
            \put(28, 20){\zoom{\cirthree}}
            \put(2.2, 17.5){\zoom{client gets $\Z'_m$}}
            \put(2.2, 12){\zoom{sends $C(\hJ[k]_m),$}}
            \put(2.2, 7){\zoom{$\neural[k]_{m}(\Z'_m), \Y_{m}$ }}
    
            \put(27, 73){\zoom{\cirfour}}
            \put(1.6, 69){\zoom{shuffling server  }}
            \put(1.6, 65){\zoom{performs}}
            \put(1.6, 61){\zoom{permutation}}
    
            \put(66.5, 73){\zoom{\cirfive}}
            \put(71, 70){\zoom{aggregagtion }}
            \put(71, 66){\zoom{server builds}}
            \put(71, 62){\zoom{kernel $\kernel[k]$ \& }}
            \put(71, 58){\zoom{obtains $\Delta \w[k]$}}
        \end{overpic}\\[-3pt]
        \caption{
        Schematic of CP-NTK-FL.
        A trusted key server (orange) sends an encrypted seed $\text{E}(\text{k}^+_m, \rho)$ with the public key $\text{k}^+_m$ for random projection. 
        The client transmits the required message to the shuffling server (blue) to perform a permutation.
        \label{subfig:cp-ntk-fl}
        }
    \end{minipage}
    \hfill
    \begin{minipage}[b]{0.48\textwidth}
    \centering
    \includegraphics[width=\linewidth]{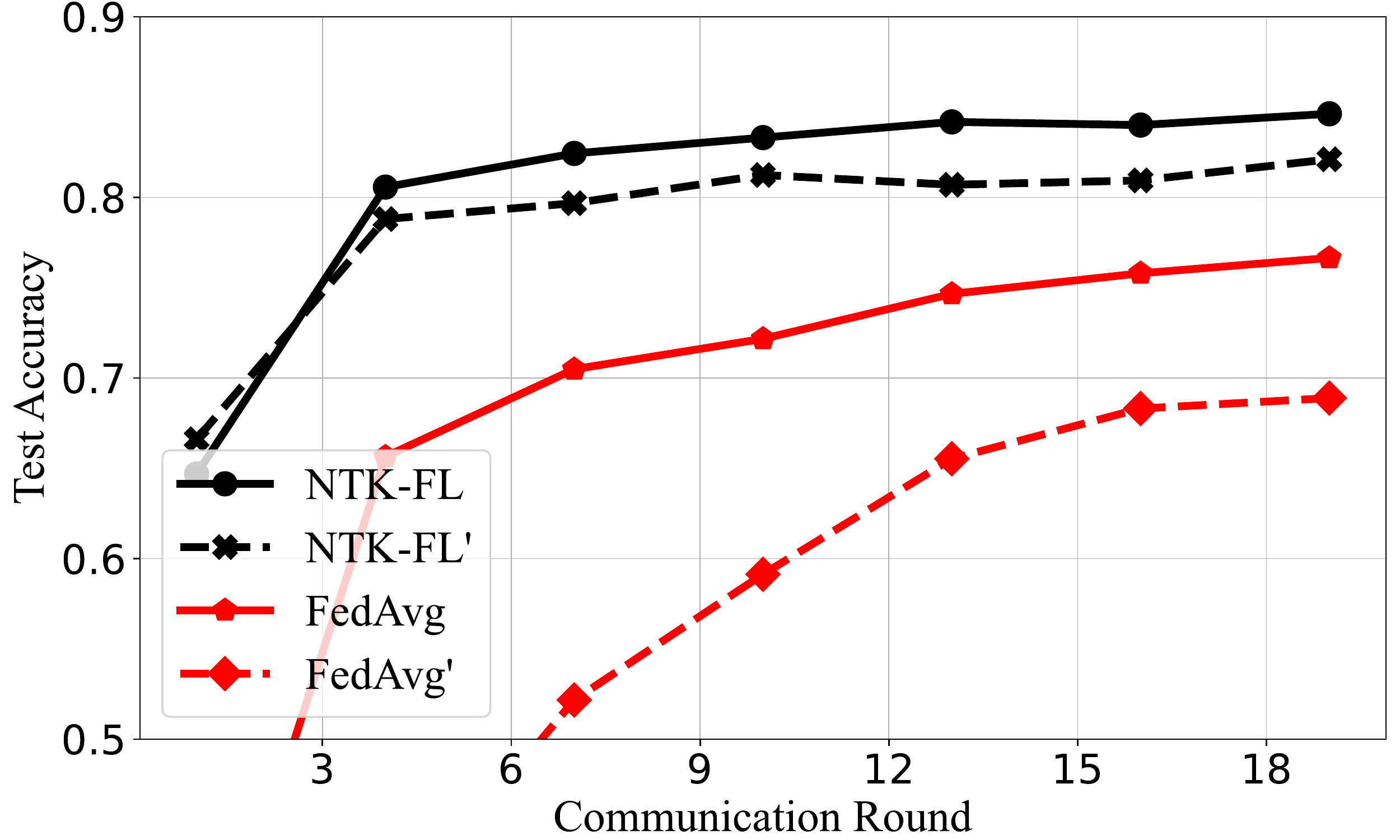}
    \caption{
    Training results of $300$ clients via NTK-FL and FedAvg, along with variants with the local dataset subsampling and random projection, denoted as NTK-FL$'$ and FedAvg$'$, respectively.  
    We train a two-layer multilayer perceptron on the Fashion-MNIST dataset.
    The joint effect causes more accuracy degradation in FedAvg (red) than in NTK-FL (black).
    \label{fig:extensions}
    }
    \end{minipage}
\end{figure*}

\subsection{CP-NTK-FL Variant}
Compared to FedAvg, NTK-FL does not incur additional client computational overhead since calculating the Jacobian tensor enjoys the same computation efficiency as computing aggregated gradients. 
Without locally updating weight vectors, NTK-FL is faster than FedAvg on the client side.
In this section, we focus on the perspectives of communication efficiency and security in terms of data confidentiality and membership privacy.

For communication, we follow the widely adopted analysis framework in wireless communication to examine only the client uplink overhead, assuming that the downlink bandwidth is much larger and the server will have enough transmission power~\citep{tran2019federated}.  
In NTK-FL, the client uplink communication cost and space complexity are dominated by a third-order tensor $\J[k]_{m}$, i.e., an $\order{N_m d_2 d}$ complexity compared to $\order{d}$ in FedAvg. 
For security, we investigate a threat model where a curious server may perform membership inference attacks \citep{nasr2018machine} or model inversion/data reconstruction attacks \citep{zhu2019deep}. 
Compared to the averaged gradient, sample-wise Jacobian matrices are more expressive, which may facilitate such attacks from the aggregation server.
We extend NTK-FL by combining various tools to solve the aforementioned problems without jeopardizing the performance severely. 
These tools are optional building modules and can be adopted separately or jointly, depending on the available resources in practice. 
Although it is possible to incorporate these tools into FedAvg, we will show that overall it will lead to a more severe accuracy drop.

\highlight{Jacobian Dimension Reduction. } 
First, we let the $m$th client sample a subset $\mathcal{B}_{m}$ from its dataset $\mathcal{D}_{m}$ uniformly for the training. 
Let $\beta \in (0,1)$ denote the sampling rate, $\mathcal{B}_{m}$ contains $N'_m = \beta N_m$ data points, with the training pairs denoted by $(\X'_m, \Y'_m)$. 
In general, using more data points will improve the model generalizability \citep{mohri2018foundations}.
The sampling rate $\beta$ controls the trade-off between efficiency and model performance. 
Next, we consider using a random projection to preprocess the input data via a seed shared by a \emph{trusted key server}. 
Formally, the sampled training examples are  projected into $\Z'_m$, i.e., 
$
    \Z'_m = \X'_m \mathbf{P},
$
where $\mathbf{P} \in \mathbb{R}^{d_1 \times d'_1}$ is a projection matrix generated based on a seed $\rho$ with IID standard Gaussian entries.
In general, we have $d_1' < d_1$ and an non-invertible projection operation. 
The concept of trusted key server follows the trusted third party in cryptography \citep{van2020computer}, and we assume it will not be compromised. 

These two steps can already improve communication efficiency and confidentiality. 
We first examine the current Jacobian tensor $\hJ[k]_m \in \mathbb{R}^{N'_m \times d_2 \times d'}$. 
Compared with its original version $\J[k]_m$, it has reduced dimensionality at the cost of certain information loss. 
Meanwhile, the random projection will defend against the data reconstruction attack, as the Jacobian tensor is now evaluated at the projected data $\Z'_m$. 
We empirically verify their impact on the test accuracy in \paperfig{fig:extensions}.
We set $d'_1 = 100$ and sampling rate $\beta = 0.4$, and train a multilayer perceptron with $100$ hidden nodes on the Fashion-MNIST dataset \citep{xiao2017fashion}. 
The joint effect of these strategies is a slight accuracy drop in NTK-FL and a nonnegligible accuracy degradation in FedAvg. 

\highlight{Jacobian Compression and Shuffling. }
We use a compression scheme to reduce the size of the Jacobian tensor by zeroing out the coordinates with small magnitudes \citep{alistarh2018the}. 
In addition to the communication efficiency, this compression scheme is empirically effective against the data reconstruction attack \citep{zhu2019deep}. 
To further ensure the confidentiality and membership privacy, we introduce a \textit{shuffling server}, inspired by some recent frameworks \citep{girgis2021shuffled, cheng2021separation}, to permute Jacobian tensors $\J[k]_m$'s, neural network states $\neural[k]_m$'s, and labels $\Y_m$'s.
Based on \eqref{eq:w_evolution}, we denote the model update by $\Delta \w[k] \triangleq \w[k+1] - \w[k] = \sum_{j=1}^{d_{2}}  (\J[k]_{:j:})^\top \res[k,t]_{:j}$, which is a sum of matrix products. 
If rows and columns are permuted in synchronization, the weight update $\Delta \w[k]$ will remain unchanged.
Considering the high dimensionality of the neural network weight, the reconstruction attack becomes computationally infeasible.
Since introducing a new differential privacy mechanism is not the main focus of this work and the privacy protection analysis is consistent with the existing framework~\citep{girgis2021shuffled}, we do not intend to go into details.  
Meanwhile, as a provable differential privacy guarantee does not explicitly protect against data reconstruction attacks~\citep{zhang2020secret}, we empirically study the privacy loss under the deep leakage from gradient algorithm~\citep{zhu2019deep} in Appendix~\ref{section:attack}.  
A thorough study of different attack schemes is out of the scope of this work and we leave it for future work.

\section{Analysis of Algorithm}
In this section, we analyze the loss decay rate between successive communication rounds in NTK-FL and make comparisons with FedAvg. 
Similar to prior work \citep{du2019gradient, dukler2020optimization}, we consider a two-layer neural network  $f:\mathbb{R}^{d} \rightarrow \mathbb{R}$ of the following form to facilitate our analysis: 
\begin{equation}\label{eq:relunet_func}
    f(\w; \x) = \frac{1}{\sqrt{n}} \sum_{r=1}^n  c_r \sigma( \vecv_{r}^\top \x ), 
\end{equation}
where $\x \in \mathbb{R}^{d_1}$ is an input, $\vecv_{r} \in \mathbb{R}^{d_1}$ is the weight vector in the first layer, $c_r$ is the weight in the second layer, and $\sigma(\cdot)$ is the rectified linear unit (ReLU) function, namely $\sigma(z) = \max(z,\,0)$, applied coordinatewise. 
\revKY{Without loss of generality, we assume the selected client set $\cset_k$ is $\{1, 2, \dots, M_{k}\} $ in communication round $k$. 
Let $\X^{(k)} = [\X^\top_1, \dots, \X^\top_{M_k}]^\top \in \mathbb{R}^{N_{k} \times d_1}$ denote a concatenation of client inputs and $\y^{(k)} = [\y^\top_1, \dots, \y^\top_{M_k}]^\top \in \mathbb{R}^{N_k \times d_2}$ denote a concatenation of client labels.  
In the following analysis, we assume $d_2 = 1$ for simplicity. }
We state two assumptions as prerequisites. 
\begin{Assumption}\label{assumption:initialization} 
    (Weight Distribution).
    \revKY{When broadcast in communication round $k$, 
    the first layer  $\vecv^{(k)}_{r}$'s follow the normal distribution $\mathcal{N}(0, \Sigma_k)$.
    The minimum eigenvalue of the covariance matrix is bounded by $\lambda_{\min}(\Sigma_k) \geqslant \alpha^2_k$, 
    where $\alpha_k$ is a positive constant. }
    The second layer  $c_r$'s are sampled from $\{-1,1\}$ with equal probability and are kept fixed during training. 
\end{Assumption}
\begin{Assumption}\label{assumption:input}
    (Normalized input). The input data are normalized, i.e., $\|\x_i\|_2 = 1, \forall\; i \in [N_k]$. 
\end{Assumption}
For this neural network model, the $(i,j)$th entry of the empirical kernel matrix $\kernel[k]$ given in \eqref{eq:kernel_entry} can be calculated as:
$
    (\kernel[k])_{ij} = \frac{1}{n} \x_i^\top \x_j \sum_{r=1}^n \indi[k]_{ir} \indi[k]_{jr},
$
where $\indi[k]_{ir} \triangleq \ind\{ \langle\vecv[k]_{r}, \x_i \rangle \geqslant 0\}$, and the term $c^2_r$ is omitted according to \paperassumption{assumption:initialization}.    
\revKY{Define $\kernelinf$, whose $(i, j)$th entry is given by:}
\begin{equation}
    \!(\kernelinf)_{ij} \triangleq \expect_{\vecv[k]}\! \left[\x^\top_i \x_j \ind^{(k)}_i \!  \ind^{(k)}_j \right],\!
\end{equation}
\revKY{where $\ind^{(k)}_i \triangleq  \ind(\langle \vecv[k], \x_i \rangle \geqslant 0)$ and $\vecv[k]$ follows the normal distribution $\mathcal{N}(0, \Sigma_k)$.}
Let $\lambda_{k}$ denote the minimum eigenvalue of $\kernelinf$, which is restricted as follows.
\begin{Assumption}\label{assumption:pdkernel}
    The kernel matrix $\kernelinf$ is positive definite, namely, $\lambda_{k} > 0$. 
\end{Assumption}
In fact, the positive-definite property of  $\kernelinf$ can be shown under certain conditions \citep{dukler2020optimization}. 
For simplicity, we omit the proof details and directly assume the positive definiteness of $\kernelinf$ in \paperassumption{assumption:pdkernel}. 
\revKY{Next, we study the residual term $\|f^{(k)}(\X^{(k)})-\y^{(k)}\|_{2}^{2}$.}
We give the convergence result by analyzing how the residual term decays given training examples $\X^{(k)}$.
\begin{Theorem}\label{th:ntkfl}
    For the NTK-FL scheme under Assumptions \ref{assumption:initialization} to \ref{assumption:pdkernel}, 
    let the learning rate $\eta=\order{\frac{\lambda_{k}}{N_k}}$ and the neural network width $n = \Omega\left(\frac{ N_k^{2}}{\lambda_{k}^{2}} \ln \frac{ N_k^{2}}{\delta} \right)$, 
    then with probability at least $1-\delta$,
    \begin{multline}
        \normsq{f^{(k+1)}(\X^{(k)})-\y^{(k)}}  \leqslant \\  \left(1\!-\!\frac{\eta \lambda_{k} }{2 N_k} \right)^{t^{(k)}} \! \normsq{f^{(k)}(\X^{(k)}) - \mathbf{y}^{(k)}}, 
    \end{multline}
    where $t^{(k)}$ is the number of update steps defined in \eqref{eq:t_optim}.
\end{Theorem}
The proof of \papertheorem{th:ntkfl} can be found in Appendix \ref{section:proof_ntkfl}. 
\revKY{We discuss the choice of the optimal number of update steps $t^{(k)}$ below. }
\begin{Remark}
\revKY{Based on the solution given by \eqref{eq:f_evolution}, the neural network prediction error is a decreasing function of the update steps $t$. 
However, one could not pick an arbitrarily large $t$ as the final optimal number of update steps $t^{(k)}$. 
When $t$ increases, the neural network evolution in the function space does not consistently match with the evolution in the weight space. 
Empirical studies in the centralized training have confirmed the nonnegligible gap between the NTK weight and gradient descent weight when $t$ is greater than the order of $10^2$ to $10^3$ \citep{lee2019wide}.   
In \paperlemma{lemma:gap_ntk_gd} of \paperappendix{section:update_step}, we provide detailed explanations and show that the difference between the NTK weight in \eqref{eq:w_evolution} and the gradient descent weight increases with $t$.}
\end{Remark}

Next, we compare the proposed NTK-FL with FedAvg. 
By studying the asymmetric kernel matrix caused by local update \citep{huang2021fl}, we have the following theorem for FedAvg, where the proof can be found in Appendix \ref{section:proof_fedavg}.

\begin{figure*}[tb]
    \centering
    \subcaptionbox{}[0.325\textwidth]{
        \includegraphics[width=\linewidth]{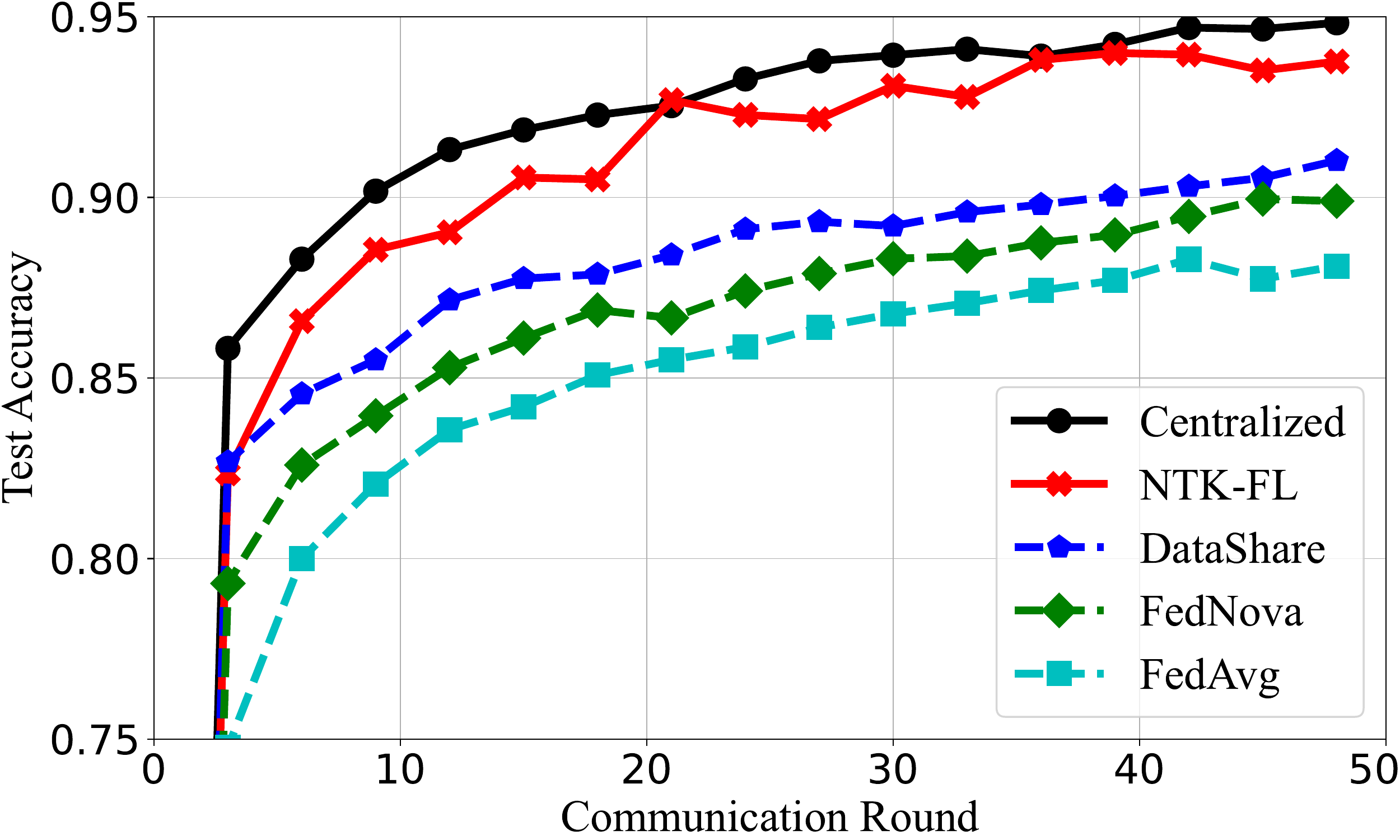}
    }
    \subcaptionbox{}[0.325\textwidth]{
        \includegraphics[width=\linewidth]{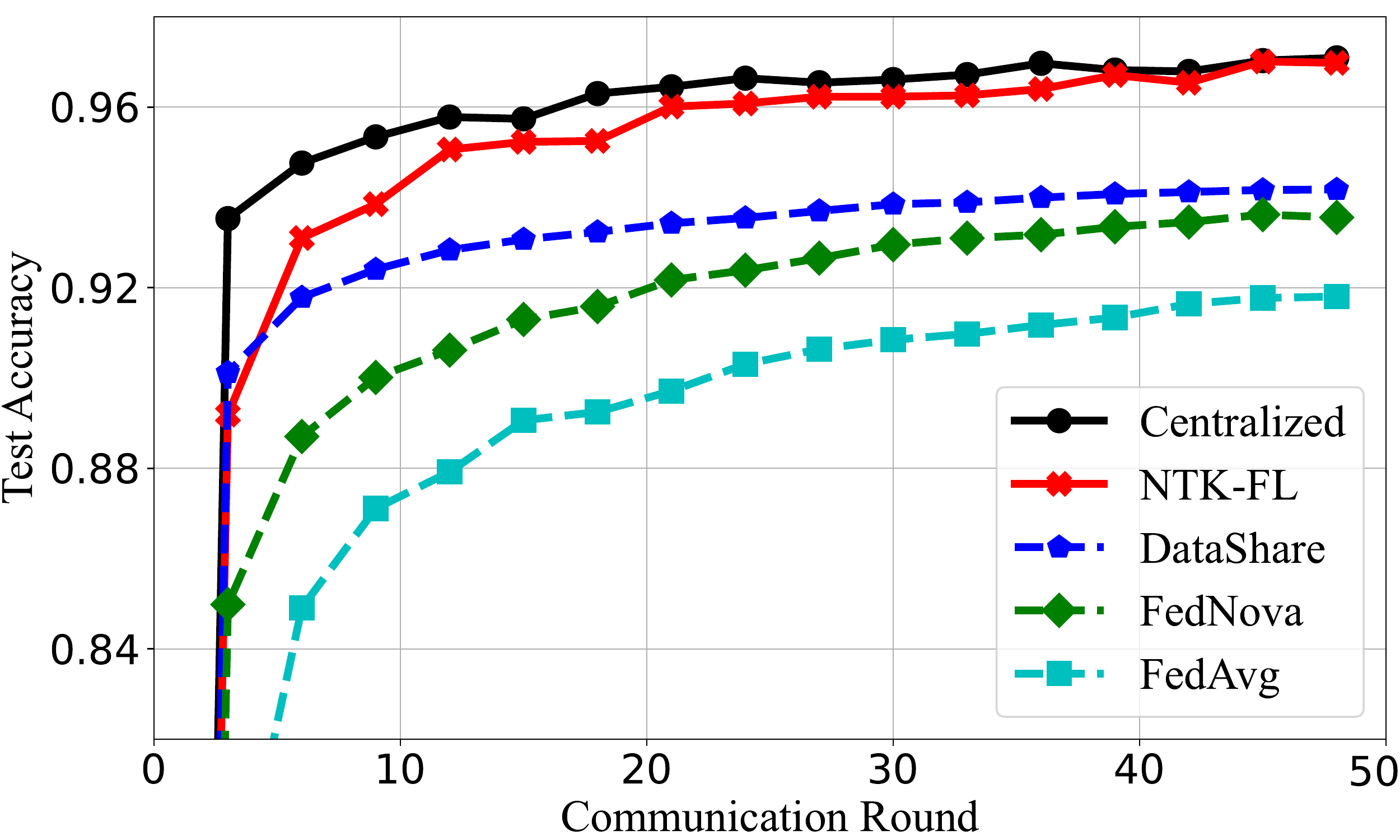}
    }
    \subcaptionbox{}[0.325\textwidth]{
        \includegraphics[width=\linewidth]{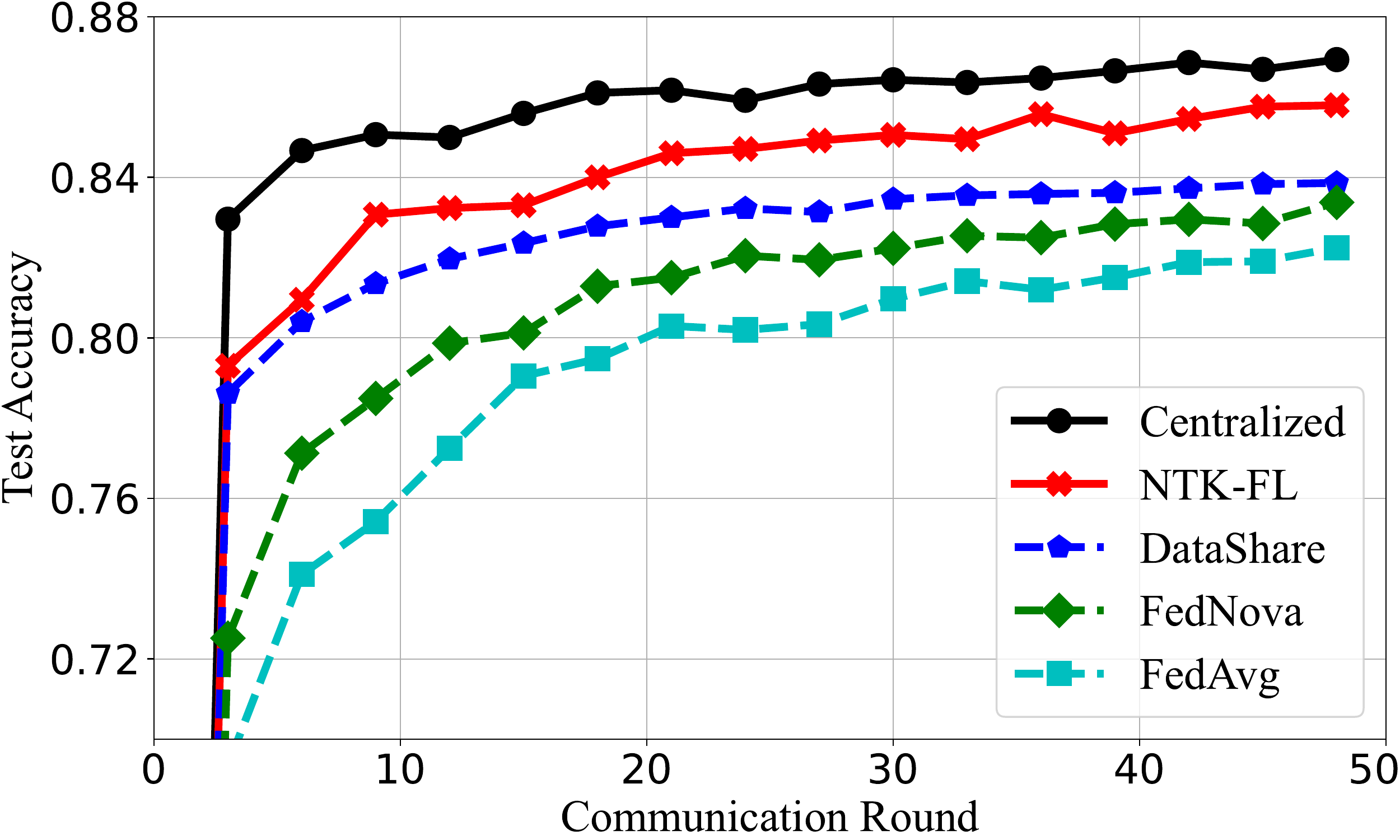}
    }

    \caption{
    Test accuracy versus communication round of different methods evaluated on: 
    (a) FEMNIST dataset, where the heterogeneity comes from feature skewness. 
    (b) non-IID MNIST dataset with label skewness ($\alpha=0.5$). 
    (c) non-IID Fashion-MNIST dataset with label skewness ($\alpha=0.5$). 
    NTK-FL outperforms all baseline FL algorithms in different scenarios and achieves similar test performance compared with the ideal centralized training case.  
    \label{fig:learning_curves}}
\end{figure*}

\begin{Theorem}\label{th:fedavg}
    For FedAvg under Assumptions \ref{assumption:initialization} to \ref{assumption:pdkernel}, 
    let the learning rate $\eta=\order{\frac{\lambda_{k}}{\tau N_k M_k}}$ and 
    the neural network width $n = \Omega\left(\frac{ N_k^{2}}{\lambda_{k}^{2}} \ln \frac{ N_k^{2}}{\delta} \right)$, 
    then with probability at least $1-\delta$, 
    \begin{multline}
        \normsq{f^{(k+1)}(\X^{(k)}) - \y^{(k)}} \leqslant \\ \left( 1 - \frac{\eta \tau \lambda_{k}}{2 N_k M_k}  \right) \normsq{f^{(k)}(\X^{(k)}) - \y^{(k)}}, 
    \end{multline}
    where $\tau$ is the number of local iterations, and $M_k$ is the number of clients in communication round $k$. 
\end{Theorem}

\begin{Remark}
    (Fast Convergence of NTK-FL). 
    The convergence rate of NTK-FL is faster than FedAvg. 
    To see this, we compare the Binomial approximation of the decay coefficient in \papertheorem{th:ntkfl} with the decay coefficient in \papertheorem{th:fedavg},
    \begin{equation}
        1 - \frac{\eta_{1} t^{(k)} \lambda_{k} }{2 N_k} + \order{\eta_1^2} <  1 - \frac{\eta_{2} \tau \lambda_{k}}{2 N_k M_k}, \label{eq:fast_convergence}
    \end{equation}
    where $\eta_1 \ll 1$ for a large $N_k$. \footnote{For example, if we have $100$ clients, each of which has $100$ data points, then $N_k$ is on the order of $10^4$. 
	Considering the choice of the learning rate $\eta_1$, the Binomial approximation holds.} 
    The number of NTK update steps $t^{(k)}$ is chosen dynamically in \eqref{eq:t_optim}, which is on the order of $10^2$ to $10^3$, whereas $\tau$ is often on the order of magnitude of $10$ in the literature \citep{reisizadeh2020fedpaq, haddadpour2021federated}.
    One can verify that $\eta_{1} t^{(k)} \lambda_{k}$ is larger than $\eta_{2} \tau \lambda_{k}/M_k$ and draw the conclusion.
\end{Remark}
\begin{figure*}
    \centering
    \begin{minipage}[b]{0.51\textwidth}
      \includegraphics[width=\textwidth]{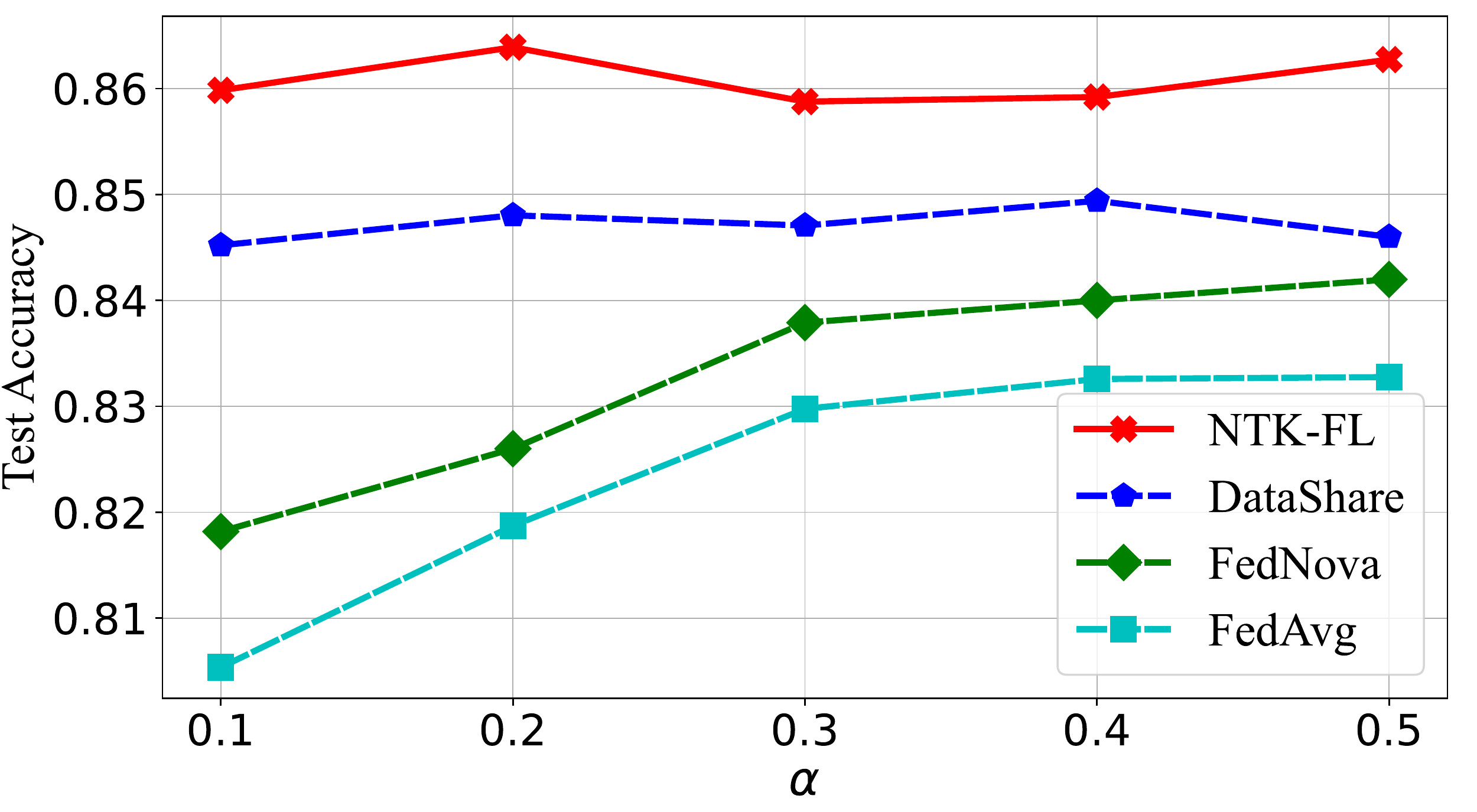}
      \caption{Test accuracy versus the Dirichlet distribution parameter $\alpha$ for different methods evaluated on the non-IID Fashion-MNIST dataset. 
          Reducing the value of $\alpha$ will increase the degree of heterogeneity in the data distribution. 
          NTK-FL is robust to different heterogeneous data distributions, and shows more advantages over FedAvg and FedNova when the degree of heterogeneity is larger.
          \label{fig:acc_alpha}}
    \end{minipage}
    \hfill
    \begin{minipage}[b]{0.45\textwidth}
      \includegraphics[width=\textwidth]{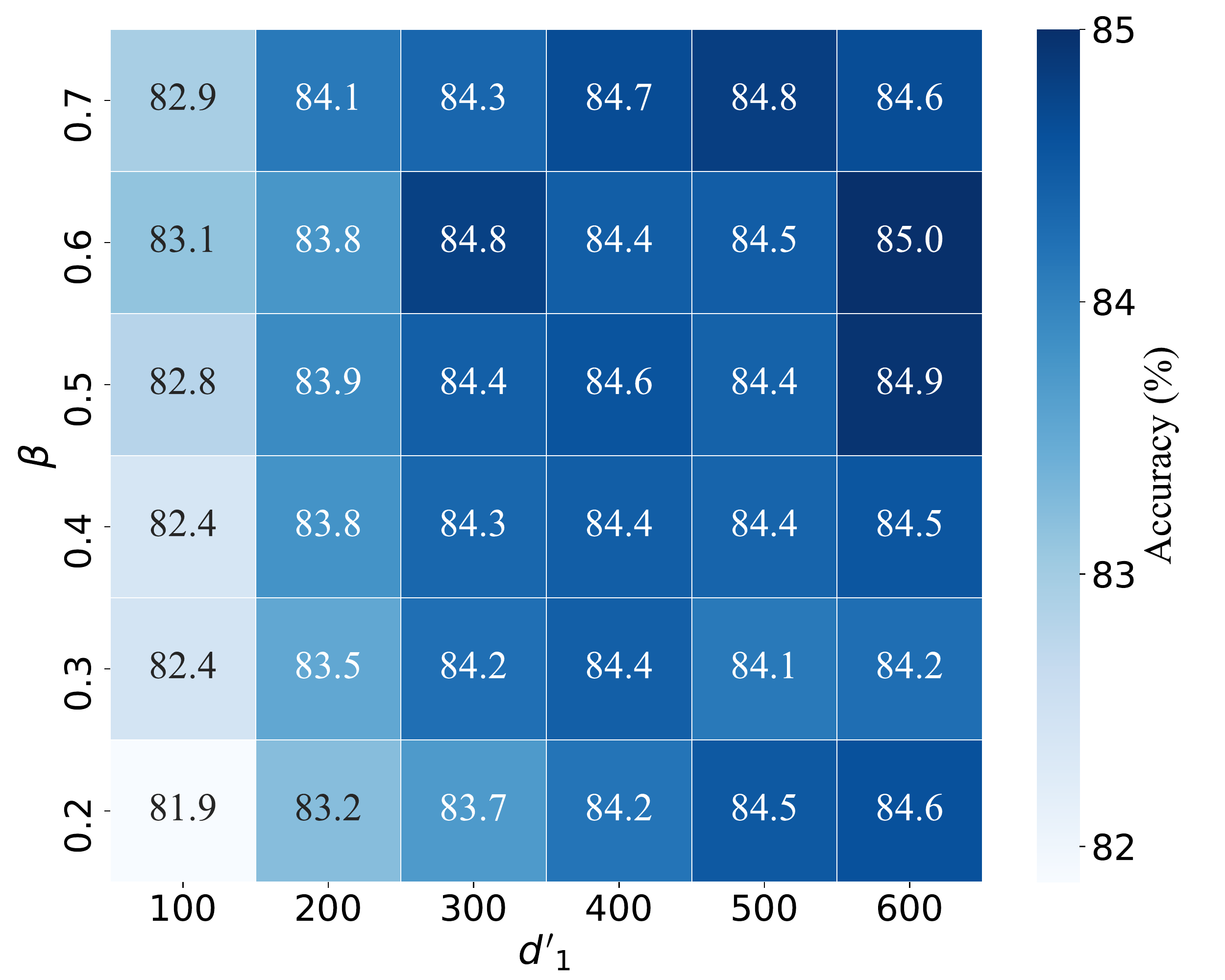}
      \caption{CP-NTK-FL test accuracy for different hyperparams. 
              A larger data sampling rate $\beta$ and a larger dimension $d'_1$ are expected to give a higher test accuracy.
              In general, the scheme is robust to different combinations of hyperparameters. 
              \label{fig:hyperparams}}
    \end{minipage}
\end{figure*}

\section{Experimental Results}\label{section:experiments}

\highlight{Federated Settings. } 
We use three datasets, namely, MNIST~\citep{lecun1998gradient}, Fashion-MNIST~\citep{xiao2017fashion}, and FEMNIST~\citep{caldas2018leaf} digits. 
All of them contain $C = 10$ categories. 
For MNIST and Fashion-MNIST, we follow~\cite{hsu2019measuring} to simulate non-IID data with the symmetric Dirichlet distribution \citep{good1976application}. 
Specifically, for the $m$th client, we draw a random vector $\boldsymbol{q}_m \sim \text{Dir}(\alpha)$, 
where $\boldsymbol{q}_m = [q_{m,1}, \dots, q_{m,C}]^{\top}$ belongs to the $(C-1)$-standard simplex.  
Images with category $k$ are assigned to the $m$th client in proportional to $(100 \cdot q_{m,k})\%$. 
The heterogeneity in this setting mainly comes from label skewness.  
FEMNIST splits the dataset into shards  indexed by the original writer of the digits. 
The heterogeneity mainly comes from feature skewness. 
A multilayer perceptron model with 100 hidden nodes is chosen as the target neural network model.
We consider a total of $300$ clients and select $20$ of them with equal probability in each round.  

\highlight{Convergence. }
We empirically verify the convergence rate of the proposed method. 
For FedAvg, we use the number of local iterations from $\{1,3, \dots ,9,10,20,\dots,50\}$ and report the best results. 
For NTK-FL, we choose $t^{(k)}$ over the set $\{100, 200, \dots, 2000\}$.
We use the following methods that are robust to the non-IID setting as the baselines: 
(i)~Data sharing scheme suggested by ~\cite{zhao2018federated}, where a global dataset is broadcast to clients for local training; 
the size of the global dataset is set to be $10\%$ of the total number of local data points. 
(ii)~Federated normalized averaging (FedNova)~\citep{wang2020tackling}, where the clients transmit normalized gradient vectors to the server. 
(iii)~Centralized training simulation, where the server collects the data points from subset $\cset_k$ of clients and performs gradient descent to directly train the global model. 
Scheme~(iii) achieves the performance that can be considered as an upper bound of all other algorithms. 
The training curves over three repetitions are shown in \paperfig{fig:learning_curves}. 
More implementation details and the results on CIFAR-10 can be found in \paperappendix{section:cifar_experiments}.
Our proposed NTK-FL method shows consistent advantages over other methods in different non-IID scenarios.

\highlight{Degree of Heterogeneity. } 
In this experiment, we select the Dirichlet distribution parameter $\alpha$ from $\{0.1, 0.2, 0.3, 0.4, 0.5\}$ and simulate different degrees of heterogeneity on Fashion-MNIST dataset. 
A smaller $\alpha$ will increase the degree of heterogeneity in the data distribution.
We evaluate NTK-FL, DataShare, FedNova, and FedAvg model test accuracy after training for $50$ rounds.
The mean values over three repetitions are shown in \paperfig{fig:acc_alpha}, where each point is obtained over five repetitions with a standard deviation less than $1\%$. 
It can be observed that NTK-FL achieves stable test accuracy in different heterogeneous settings.
In comparison, FedAvg and FedNova show a performance drop in the small $\alpha$ region. 
NTK-FL has more advantages over baselines methods when the degree of heterogeneity is larger.  

\highlight{Effect of Hyperparameters. }
We study the effect of the tunable parameters in CP-NTK-FL.  
We change the local data sampling rate $\beta$ and dimension $d'_1$, and evaluate the model test accuracy on the non-IID Fashion-MNIST dataset ($\alpha=0.1$) after $10$ communication rounds. 
The results are shown in \paperfig{fig:hyperparams}.
A larger data sampling rate $\beta$ or a larger dimension $d'_1$ will cause less information loss and are expected to achieve a higher test accuracy.
The results also show that the scheme is robust to different combinations of hyperparameters.

\highlight{Uplink Communication. }
In federated learning, uplink communication overhead can be one of the bottlenecks in the training stage. 
We evaluate the uplink communication efficiency of CP-NTK-FL ($d'_1\!=\!200, \beta\!=\!0.3$) by measuring the number of rounds and cumulative uplink communication cost to reach a test accuracy of $85\%$ on non-IID Fashion-MNIST dataset ($\alpha=0.1$).  
The results over three repetitions are shown in \papertab{tab:comm_efficiency}.
Compared with federated learning with compression  (FedCOM) \citep{haddadpour2021federated}, quantized SGD (QSGD) \citep{alistarh2017qsgd}, and FedAvg,  
CP-NTK-FL achieves the goal within an order of magnitude fewer iterations, which is particularly advantageous for applications with nonnegligible encoding/decoding delays or network latency. 

\begin{table}
    \centering
    \caption{Uplink communication cost to reach an accuracy of $85\%$ on non-IID Fashion-MNIST dataset ($\alpha=0.1$) for different methods. 
    CP-NTK-FL can achieve the target within the fewest communication rounds without incurring communication cost significantly. 
    \label{tab:comm_efficiency}}
    \begin{tabular}{p{0.31\linewidth}p{0.23\linewidth}p{0.23\linewidth}}
        \doublerule \\[-18pt]
        optimization algorithms & comm. rounds & comm. cost (MB) \\ \midrule
        CP-NTK-FL & $26$ & $386$ \\ \cmidrule{1-3} 
        FedCOM & $250$ & $379$ \\[3pt] 
        QSGD (4 bit) & $614$ & $465$ \\[3pt]
        FedAvg & $284$ & $1720$  \\\bottomrule
    \end{tabular}
\end{table}

\section{Conclusion and Future Work}
In this paper, we have proposed an NTK empowered FL paradigm. 
It inherently solves the statistical heterogeneity challenge. 
By constructing a global kernel based on the local sample-wise Jacobian matrices, the global model weights can be updated via NTK evolution in the parameter space. 
Compared with traditional algorithms such as FedAvg, NTK-FL has a more centralized training flavor by transmitting more expressive updates.
The effectiveness of the proposed paradigm has been verified theoretically and experimentally. 

In future work, it will be interesting to extend the paradigm for other neural network architectures, such as CNNs, residual networks (ResNets) \citep{he2016deep}, and RNNs. 
It is also worthwhile to further improve the efficiency of NTK-FL and explore its savings in wall-clock time. 
We believe the proposed paradigm will provide a new perspective to solve  federated learning challenges. 

\bibliographystyle{icml}
\bibliography{myref}

\begin{thebibliography}{50}
\providecommand{\natexlab}[1]{#1}
\providecommand{\url}[1]{\texttt{#1}}
\expandafter\ifx\csname urlstyle\endcsname\relax
  \providecommand{\doi}[1]{doi: #1}\else
  \providecommand{\doi}{doi: \begingroup \urlstyle{rm}\Url}\fi

\bibitem[{Aji} \& {Heafield}(2017){Aji} and {Heafield}]{aji2017sparse}
{Aji}, A.~F. and {Heafield}, K.
\newblock Sparse communication for distributed gradient descent.
\newblock In \emph{Proceedings of the 2017 Conference on Empirical Methods in
  Natural Language Processing}, 2017.

\bibitem[{Alemohammad} et~al.(2021){Alemohammad}, {Wang}, {Balestriero}, and
  {Baraniuk}]{alemohammad2021the}
{Alemohammad}, S., {Wang}, Z., {Balestriero}, R., and {Baraniuk}, R.
\newblock The recurrent neural tangent kernel.
\newblock In \emph{International Conference on Learning Representations}, 2021.

\bibitem[Alistarh et~al.(2017)Alistarh, Grubic, Li, Tomioka, and
  Vojnovic]{alistarh2017qsgd}
Alistarh, D., Grubic, D., Li, J., Tomioka, R., and Vojnovic, M.
\newblock {QSGD}: Communication-efficient {SGD} via gradient quantization and
  encoding.
\newblock 2017.

\bibitem[{Alistarh} et~al.(2018){Alistarh}, {Hoefler}, {Johansson},
  {Konstantinov}, {Khirirat}, and {Renggli}]{alistarh2018the}
{Alistarh}, D., {Hoefler}, T., {Johansson}, M., {Konstantinov}, N., {Khirirat},
  S., and {Renggli}, C.
\newblock The convergence of sparsified gradient methods.
\newblock In \emph{Advances in Neural Information Processing Systems}, 2018.

\bibitem[{Arora} et~al.(2019){Arora}, {Du}, {Hu}, {Li}, {Salakhutdinov}, and
  {Wang}]{arora2019on}
{Arora}, S., {Du}, S.~S., {Hu}, W., {Li}, Z., {Salakhutdinov}, R., and {Wang},
  R.
\newblock On exact computation with an infinitely wide neural net.
\newblock In \emph{Advances in Neural Information Processing Systems}, 2019.

\bibitem[Caldas et~al.(2018)Caldas, Duddu, Wu, Li, Kone{\v{c}}n{\`y}, McMahan,
  Smith, and Talwalkar]{caldas2018leaf}
Caldas, S., Duddu, S. M.~K., Wu, P., Li, T., Kone{\v{c}}n{\`y}, J., McMahan,
  H.~B., Smith, V., and Talwalkar, A.
\newblock Leaf: A benchmark for federated settings.
\newblock \emph{arXiv preprint arXiv:1812.01097}, 2018.

\bibitem[Chen \& Chao(2021)Chen and Chao]{chen2021fedbe}
Chen, H.-Y. and Chao, W.-L.
\newblock Fed{BE}: Making {B}ayesian model ensemble applicable to federated
  learning.
\newblock In \emph{International Conference on Learning Representations}, 2021.

\bibitem[{Chen} et~al.(2020){Chen}, {Cao}, {Gu}, and {Zhang}]{chen2020a}
{Chen}, Z., {Cao}, Y., {Gu}, Q., and {Zhang}, T.
\newblock A generalized neural tangent kernel analysis for two-layer neural
  networks.
\newblock In \emph{Advances in Neural Information Processing Systems}, 2020.

\bibitem[Cheng et~al.(2021)Cheng, Eykholt, Gu, Jamjoom, Jayaram, Valdez, and
  Verma]{cheng2021separation}
Cheng, P.-C., Eykholt, K., Gu, Z., Jamjoom, H., Jayaram, K., Valdez, E., and
  Verma, A.
\newblock Separation of powers in federated learning.
\newblock \emph{arXiv preprint arXiv:2105.09400}, 2021.

\bibitem[{Du} et~al.(2019){Du}, {Zhai}, {Poczos}, and {Singh}]{du2019gradient}
{Du}, S.~S., {Zhai}, X., {Poczos}, B., and {Singh}, A.
\newblock Gradient descent provably optimizes over-parameterized neural
  networks.
\newblock In \emph{International Conference on Learning Representations}, 2019.

\bibitem[{Dukler} et~al.(2020){Dukler}, {Montufar}, and
  {Gu}]{dukler2020optimization}
{Dukler}, Y., {Montufar}, G., and {Gu}, Q.
\newblock Optimization theory for {R}e{LU} neural networks trained with
  normalization layers.
\newblock In \emph{International Conference on Machine Learning}, 2020.

\bibitem[{Girgis} et~al.(2021){Girgis}, {Data}, {Diggavi}, {Kairouz}, and
  {Suresh}]{girgis2021shuffled}
{Girgis}, A.~M., {Data}, D., {Diggavi}, S.~N., {Kairouz}, P., and {Suresh},
  A.~T.
\newblock Shuffled model of differential privacy in federated learning.
\newblock In \emph{International Conference on Artificial Intelligence and
  Statistics}, 2021.

\bibitem[{Gonzalez} \& {Woods}(2014){Gonzalez} and
  {Woods}]{gonzalez2014digital}
{Gonzalez}, R.~C. and {Woods}, R.~E.
\newblock \emph{Digital {I}mage {P}rocessing, 3rd {E}dition}.
\newblock 2014.

\bibitem[Good(1976)]{good1976application}
Good, I.~J.
\newblock On the application of symmetric {D}irichlet distributions and their
  mixtures to contingency tables.
\newblock \emph{The Annals of Statistics}, 4\penalty0 (6):\penalty0 1159--1189,
  1976.

\bibitem[Haddadpour et~al.(2021)Haddadpour, Kamani, Mokhtari, and
  Mahdavi]{haddadpour2021federated}
Haddadpour, F., Kamani, M.~M., Mokhtari, A., and Mahdavi, M.
\newblock Federated learning with compression: Unified analysis and sharp
  guarantees.
\newblock In \emph{International Conference on Artificial Intelligence and
  Statistics}, 2021.

\bibitem[{He} et~al.(2016){He}, {Zhang}, {Ren}, and {Sun}]{he2016deep}
{He}, K., {Zhang}, X., {Ren}, S., and {Sun}, J.
\newblock Deep residual learning for image recognition.
\newblock In \emph{International Conference on Computer Vision and Pattern
  Recognition}, 2016.

\bibitem[Hsu et~al.(2019)Hsu, Qi, and Brown]{hsu2019measuring}
Hsu, T.-M.~H., Qi, H., and Brown, M.
\newblock Measuring the effects of non-identical data distribution for
  federated visual classification.
\newblock \emph{arXiv preprint arXiv:1909.06335}, 2019.

\bibitem[Huang et~al.(2021)Huang, Li, Song, and Yang]{huang2021fl}
Huang, B., Li, X., Song, Z., and Yang, X.
\newblock {FL-NTK}: A neural tangent kernel-based framework for federated
  learning convergence analysis.
\newblock \emph{arXiv preprint arXiv:2105.05001}, 2021.

\bibitem[{Jacot} et~al.(2018){Jacot}, {Gabriel}, and
  {Hongler}]{jacot2018neural}
{Jacot}, A., {Gabriel}, F., and {Hongler}, C.
\newblock Neural tangent kernel: Convergence and generalization in neural
  networks.
\newblock In \emph{Advances in Neural Information Processing Systems}, 2018.

\bibitem[Kairouz et~al.(2021)Kairouz, McMahan, Avent, Bellet, Bennis,
  et~al.]{kairouz2021advances}
Kairouz, P., McMahan, H.~B., Avent, B., Bellet, A., Bennis, M., et~al.
\newblock Advances and open problems in federated learning.
\newblock \emph{Foundations and Trends in Machine Learning}, 2021.

\bibitem[Karimireddy et~al.(2020)Karimireddy, Kale, Mohri, Reddi, Stich, and
  Suresh]{karimireddy2020scaffold}
Karimireddy, S.~P., Kale, S., Mohri, M., Reddi, S., Stich, S., and Suresh,
  A.~T.
\newblock Scaffold: Stochastic controlled averaging for federated learning.
\newblock In \emph{International Conference on Machine Learning}, 2020.

\bibitem[Kolda \& Bader(2009)Kolda and Bader]{kolda2009tensor}
Kolda, T.~G. and Bader, B.~W.
\newblock Tensor decompositions and applications.
\newblock \emph{SIAM {R}eview}, 51\penalty0 (3):\penalty0 455--500, 2009.

\bibitem[{Krizhevsky}(2009)]{krizhevsky2009learning}
{Krizhevsky}, A.
\newblock Learning multiple layers of features from tiny images.
\newblock \emph{Master thesis, Dept. of Comput. Sci., Univ. of Toronto,
  Toronto, Canada}, 2009.

\bibitem[LeCun et~al.(1998)LeCun, Bottou, Bengio, and
  Haffner]{lecun1998gradient}
LeCun, Y., Bottou, L., Bengio, Y., and Haffner, P.
\newblock Gradient-based learning applied to document recognition.
\newblock \emph{Proceedings of the IEEE}, 86\penalty0 (11):\penalty0
  2278--2324, 1998.

\bibitem[Lee et~al.(2019)Lee, Xiao, Schoenholz, Bahri, Novak, Sohl-Dickstein,
  and Pennington]{lee2019wide}
Lee, J., Xiao, L., Schoenholz, S., Bahri, Y., Novak, R., Sohl-Dickstein, J.,
  and Pennington, J.
\newblock Wide neural networks of any depth evolve as linear models under
  gradient descent.
\newblock In \emph{Advances in Neural Information Processing Systems}, 2019.

\bibitem[{Lee} et~al.(2020){Lee}, {Schoenholz}, {Pennington}, {Adlam}, {Xiao},
  {Novak}, and {Sohl-Dickstein}]{lee2020finite}
{Lee}, J., {Schoenholz}, S.~S., {Pennington}, J., {Adlam}, B., {Xiao}, L.,
  {Novak}, R., and {Sohl-Dickstein}, J.
\newblock Finite versus infinite neural networks: An empirical study.
\newblock In \emph{Advances in Neural Information Processing Systems}, 2020.

\bibitem[Li et~al.(2020)Li, Sahu, Talwalkar, and Smith]{li2020flsurvey}
Li, T., Sahu, A.~K., Talwalkar, A., and Smith, V.
\newblock Federated learning: Challenges, methods, and future directions.
\newblock \emph{IEEE Signal Processing Magazine}, 37\penalty0 (3):\penalty0
  50--60, 2020.

\bibitem[{Li} et~al.(2020){Li}, {Sahu}, {Zaheer}, {Sanjabi}, {Talwalkar}, and
  {Smith}]{li2020federated}
{Li}, T., {Sahu}, A.~K., {Zaheer}, M., {Sanjabi}, M., {Talwalkar}, A., and
  {Smith}, V.
\newblock Federated optimization in heterogeneous networks.
\newblock 2020.

\bibitem[{Li} et~al.(2021){Li}, {Jiang}, {Zhang}, {Kamp}, and
  {Dou}]{li2021fedbn}
{Li}, X., {Jiang}, M., {Zhang}, X., {Kamp}, M., and {Dou}, Q.
\newblock Fed{BN}: Federated learning on non-iid features via local batch
  normalization.
\newblock In \emph{International Conference on Learning Representations}, 2021.

\bibitem[Liang et~al.(2019)Liang, Liu, Ziyin, Allen, Auerbach, Brent,
  Salakhutdinov, and Morency]{liang2019think}
Liang, P.~P., Liu, T., Ziyin, L., Allen, N.~B., Auerbach, R.~P., Brent, D.,
  Salakhutdinov, R., and Morency, L.-P.
\newblock Think locally, act globally: Federated learning with local and global
  representations.
\newblock \emph{NeurIPS Workshop on Federated Learning}, 2019.

\bibitem[McMahan et~al.(2017)McMahan, Moore, Ramage, Hampson, and
  y~Arcas]{mcmahan2017communication}
McMahan, B., Moore, E., Ramage, D., Hampson, S., and y~Arcas, B.~A.
\newblock Communication-efficient learning of deep networks from decentralized
  data.
\newblock In \emph{Artificial Intelligence and Statistics}, 2017.

\bibitem[Mohri et~al.(2018)Mohri, Rostamizadeh, and
  Talwalkar]{mohri2018foundations}
Mohri, M., Rostamizadeh, A., and Talwalkar, A.
\newblock \emph{Foundations of machine learning}.
\newblock MIT press, 2018.

\bibitem[{Nakkiran} et~al.(2020){Nakkiran}, {Kaplun}, {Bansal}, {Yang},
  {Barak}, and {Sutskever}]{nakkiran2020deep}
{Nakkiran}, P., {Kaplun}, G., {Bansal}, Y., {Yang}, T., {Barak}, B., and
  {Sutskever}, I.
\newblock Deep double descent: Where bigger models and more data hurt.
\newblock In \emph{International Conference on Learning Representations}, 2020.

\bibitem[Nasr et~al.(2018)Nasr, Shokri, and Houmansadr]{nasr2018machine}
Nasr, M., Shokri, R., and Houmansadr, A.
\newblock Machine learning with membership privacy using adversarial
  regularization.
\newblock In \emph{ACM SIGSAC Conference on Computer and Communications
  Security}, pp.\  634--646, 2018.

\bibitem[{Reddi} et~al.(2021){Reddi}, {Charles}, {Zaheer}, {Garrett}, {Rush},
  {Konečný}, {Kumar}, and {McMahan}]{reddi2021adaptive}
{Reddi}, S.~J., {Charles}, Z., {Zaheer}, M., {Garrett}, Z., {Rush}, K.,
  {Konečný}, J., {Kumar}, S., and {McMahan}, H.~B.
\newblock Adaptive federated optimization.
\newblock In \emph{International Conference on Learning Representations}, 2021.

\bibitem[Reisizadeh et~al.(2020)Reisizadeh, Mokhtari, Hassani, Jadbabaie, and
  Pedarsani]{reisizadeh2020fedpaq}
Reisizadeh, A., Mokhtari, A., Hassani, H., Jadbabaie, A., and Pedarsani, R.
\newblock Fed{PAQ}: A communication-efficient federated learning method with
  periodic averaging and quantization.
\newblock In \emph{International Conference on Artificial Intelligence and
  Statistics}, 2020.

\bibitem[Sattler et~al.(2019)Sattler, Wiedemann, M{\"u}ller, and
  Samek]{sattler2019robust}
Sattler, F., Wiedemann, S., M{\"u}ller, K.-R., and Samek, W.
\newblock Robust and communication-efficient federated learning from non-iid
  data.
\newblock \emph{IEEE Transactions on Neural Networks and Learning Systems},
  31\penalty0 (9):\penalty0 3400--3413, 2019.

\bibitem[Seo et~al.(2020)Seo, Park, Oh, Bennis, and Kim]{seo2020federated}
Seo, H., Park, J., Oh, S., Bennis, M., and Kim, S.-L.
\newblock Federated knowledge distillation.
\newblock \emph{arXiv preprint arXiv:2011.02367}, 2020.

\bibitem[{Smith} et~al.(2017){Smith}, {Chiang}, {Sanjabi}, and
  {Talwalkar}]{smith2017federated}
{Smith}, V., {Chiang}, C.-K., {Sanjabi}, M., and {Talwalkar}, A.
\newblock Federated multi-task learning.
\newblock In \emph{Advances in Neural Information Processing Systems}, 2017.

\bibitem[{Su} et~al.(2021){Su}, {Xu}, and {Yang}]{su2021achieving}
{Su}, L., {Xu}, J., and {Yang}, P.
\newblock Achieving statistical optimality of federated learning: Beyond
  stationary points.
\newblock \emph{arXiv preprint arXiv:2106.15216}, 2021.

\bibitem[Tran et~al.(2019)Tran, Bao, Zomaya, Nguyen, and
  Hong]{tran2019federated}
Tran, N.~H., Bao, W., Zomaya, A., Nguyen, M.~N., and Hong, C.~S.
\newblock Federated learning over wireless networks: Optimization model design
  and analysis.
\newblock In \emph{IEEE Conference on Computer Communications}, pp.\
  1387--1395. IEEE, 2019.

\bibitem[Van~Oorschot(2020)]{van2020computer}
Van~Oorschot, P.~C.
\newblock \emph{Computer Security and the Internet: Tools and Jewels}.
\newblock Springer Nature, 2020.

\bibitem[Wang et~al.(2020)Wang, Yurochkin, Sun, Papailiopoulos, and
  Khazaeni]{wang2020federated}
Wang, H., Yurochkin, M., Sun, Y., Papailiopoulos, D., and Khazaeni, Y.
\newblock Federated learning with matched averaging.
\newblock In \emph{International Conference on Learning Representations}, 2020.

\bibitem[{Wang} et~al.(2020){Wang}, {Liu}, {Liang}, {Joshi}, and
  {Poor}]{wang2020tackling}
{Wang}, J., {Liu}, Q., {Liang}, H., {Joshi}, G., and {Poor}, H.~V.
\newblock Tackling the objective inconsistency problem in heterogeneous
  federated optimization.
\newblock In \emph{Advances in Neural Information Processing Systems}, 2020.

\bibitem[Wang et~al.(2021)Wang, Charles, Xu, Joshi, McMahan, Al-Shedivat,
  Andrew, Avestimehr, Daly, Data, et~al.]{wang2021field}
Wang, J., Charles, Z., Xu, Z., Joshi, G., McMahan, H.~B., Al-Shedivat, M.,
  Andrew, G., Avestimehr, S., Daly, K., Data, D., et~al.
\newblock A field guide to federated optimization.
\newblock \emph{arXiv preprint arXiv:2107.06917}, 2021.

\bibitem[Xiao et~al.(2017)Xiao, Rasul, and Vollgraf]{xiao2017fashion}
Xiao, H., Rasul, K., and Vollgraf, R.
\newblock Fashion-{MNIST}: a novel image dataset for benchmarking machine
  learning algorithms.
\newblock \emph{arXiv preprint arXiv:1708.07747}, 2017.

\bibitem[{Yang} \& {Littwin}(2021){Yang} and {Littwin}]{yang2021tensor}
{Yang}, G. and {Littwin}, E.
\newblock Tensor programs {II}b: Architectural universality of neural tangent
  kernel training dynamics.
\newblock In \emph{International Conference on Machine Learning}, 2021.

\bibitem[Zhang et~al.(2020)Zhang, Jia, Pei, Wang, Li, and
  Song]{zhang2020secret}
Zhang, Y., Jia, R., Pei, H., Wang, W., Li, B., and Song, D.
\newblock The secret revealer: Generative model-inversion attacks against deep
  neural networks.
\newblock In \emph{IEEE/CVF Conference on Computer Vision and Pattern
  Recognition}, pp.\  253--261, 2020.

\bibitem[Zhao et~al.(2018)Zhao, Li, Lai, Suda, Civin, and
  Chandra]{zhao2018federated}
Zhao, Y., Li, M., Lai, L., Suda, N., Civin, D., and Chandra, V.
\newblock Federated learning with non-iid data.
\newblock \emph{arXiv preprint arXiv:1806.00582}, 2018.

\bibitem[{Zhu} et~al.(2019){Zhu}, {Liu}, and {Han}]{zhu2019deep}
{Zhu}, L., {Liu}, Z., and {Han}, S.
\newblock Deep leakage from gradients.
\newblock In \emph{Advances in Neural Information Processing Systems}, 2019.

\end{thebibliography}

\newpage
\appendix
\onecolumn
\newpage
\setcounter{Lemma}{0}
\setcounter{Theorem}{0}

\section{Implementation and Additional Results}
\begin{wrapfigure}[15]{R}{0.48\textwidth}
\centering
\vspace*{-15pt}
\includegraphics[width=\linewidth]{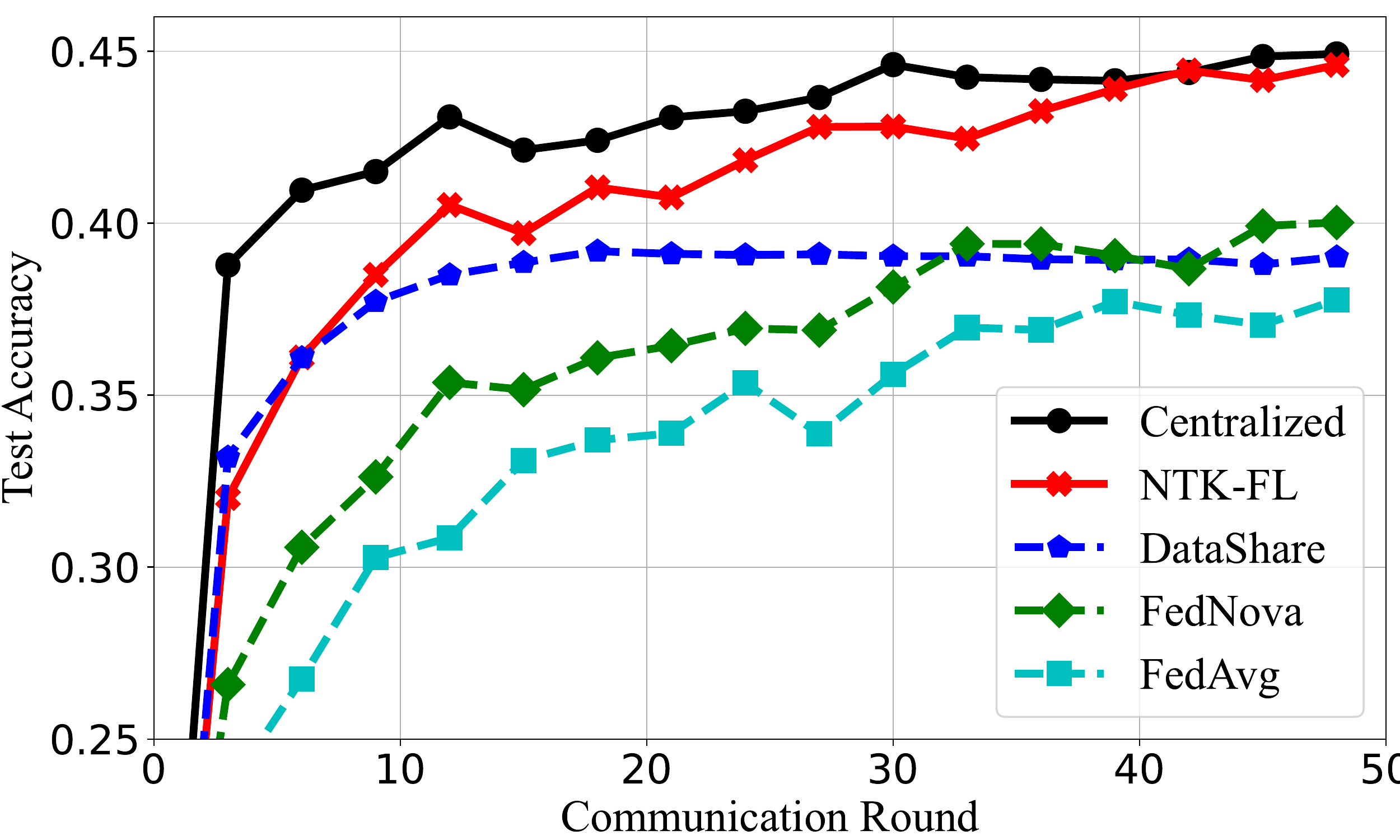}
\vspace{-15pt}
\caption{Test accuracy versus communication round of different methods evaluated on the non-IID CIFAR-10 dataset, where the Dirichlet distribution parameter $\alpha$ is set to $0.1$.  \label{fig:cifar}}
\end{wrapfigure}

We give the detailed setting of the learning rate and batch size.
For the learning rate $\eta$, we search over the set $\{10^{-3}, 3\times 10^{-3}, 10^{-2}, 3\times 10^{-2}, 10^{-1} \}$. 
The learning rate is fixed during the training. 
For the client batch size, we set it to $200$ for all datasets. 
We use the same setup as in Section~\ref{section:experiments}. 
We evaluate different methods, including the centralized training simulation, data sharing method \citep{zhao2018federated}, FedNova \citep{wang2020tackling}, FedAvg \citep{mcmahan2017communication}, and the proposed NTK-FL on the non-IID CIFAR-10 dataset \citep{krizhevsky2009learning} and present the results in \paperfig{fig:cifar}.
NTK-FL outperforms other FL algorithms and shows test accuracy close to the centralized simulation. 
The observation is consistent with the results in \paperfig{fig:learning_curves}.
The implementation is available at \url{https://github.com/KAI-YUE/ntk-fed}. 

\label{section:cifar_experiments}

\section{Reconstruction Attack}
In the following experiment, we measure the privacy loss when using Jacobian sparsification under the data reconstruction attack.
We compress the Jacobian tensor and perform the deep leakage attack~\citep{zhu2019deep}. 
The sparsity levels and image structural similarity index measure (SSIM)~\citep{gonzalez2014digital} are shown in Figure~\ref{fig:dlg}.
When the sparsity decreases, the reconstructions are closer to original images, which is consistent with~\cite{zhu2019deep}.  
The sparsity is set above $80\%$--$90\%$ when using the Top-$k$ sparsification approach~\citep{aji2017sparse}, where the privacy loss becomes difficult to quantify. 
A solid privacy-protection study is nontrivial and we leave it for future work.
\begin{figure}[!ht]
  \begin{overpic}[width=\linewidth, height=100pt]{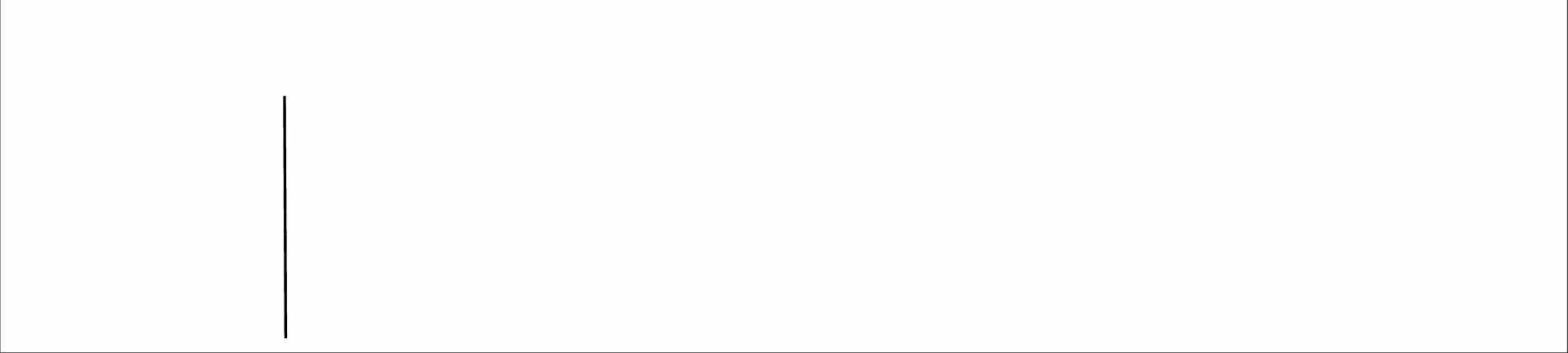}
  \put(0.5, 0.3){\includegraphics[width=0.142\textwidth]{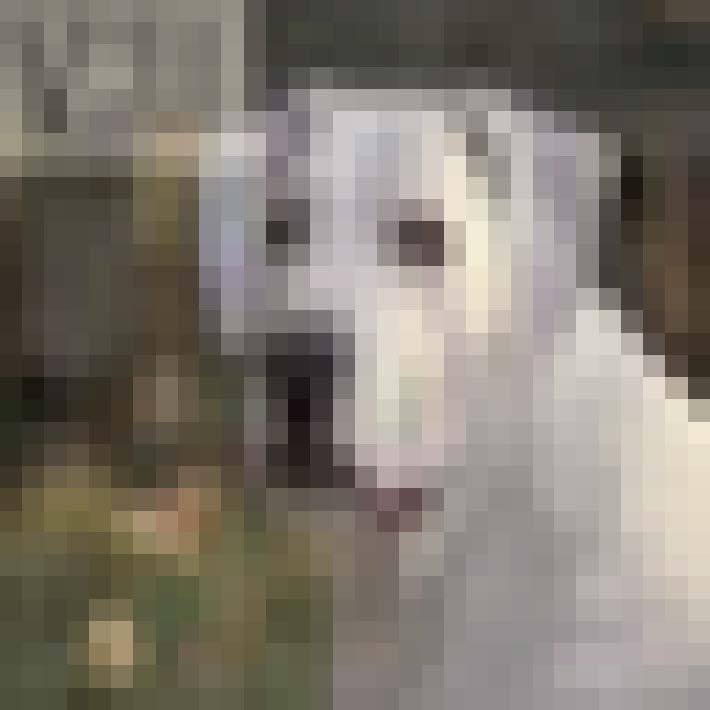}}
  \put(19.5, 0.3){\includegraphics[width=0.142\textwidth]{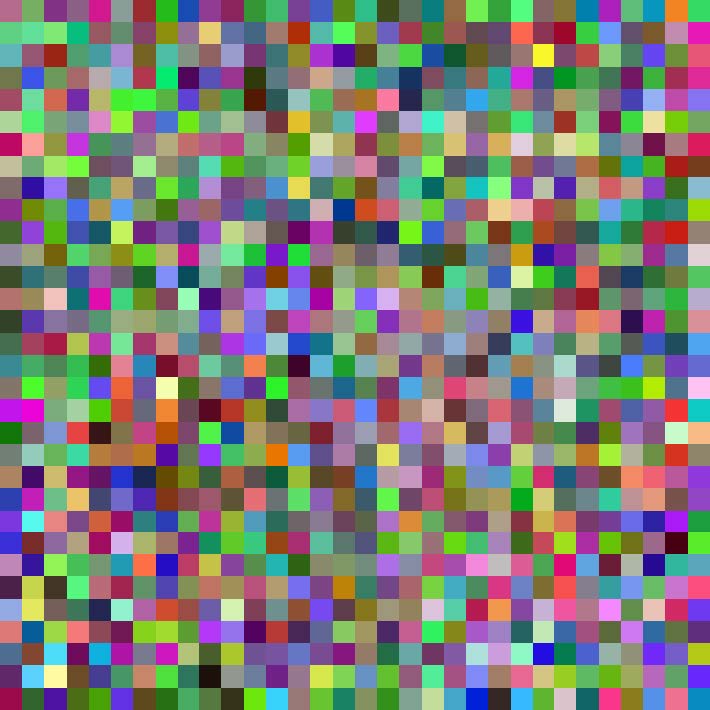}}
  \put(36, 0.3){\includegraphics[width=0.142\textwidth]{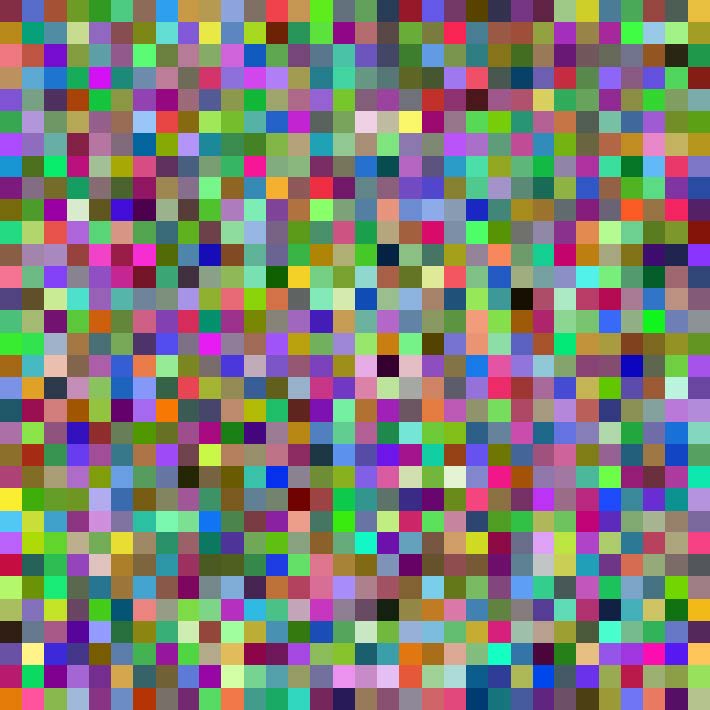}}
  \put(52.5, 0.3){\includegraphics[width=0.142\textwidth]{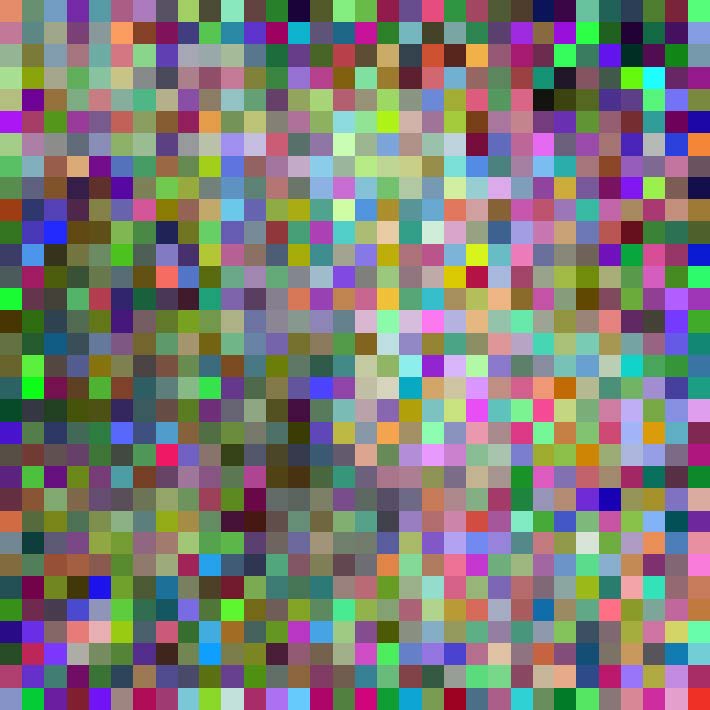}}
  \put(69, 0.3){\includegraphics[width=0.142\textwidth]{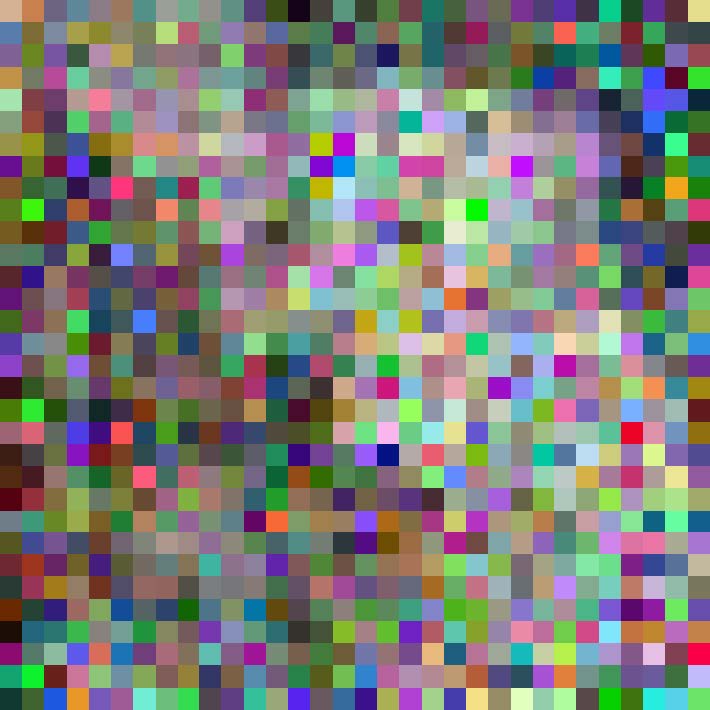}}
  \put(85.5, 0.3){\includegraphics[width=0.142\textwidth]{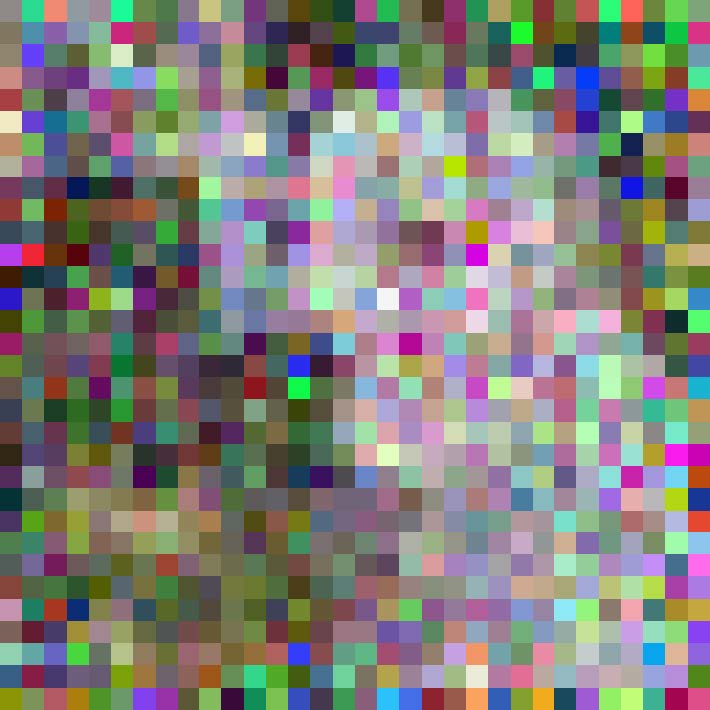}}

  \put(2.5,18){\intab[1]{original}}

  \put(14, 20){\intab[1]{Sparsity}}
  \put(15, 17){\intab[1]{SSIM}}

  \put(24.5, 20){\intab[1]{$90\%$ }}
  \put(24.5, 17){\intab[1]{$0.06$ }}

  \put(41, 20){\intab[1]{$80\%$ }}
  \put(41, 17){\intab[1]{$0.23$ }}

  \put(57, 20){\intab[1]{$70\%$ }}
  \put(57, 17){\intab[1]{$0.35$ }}

  \put(74, 20){\intab[1]{$60\%$ }}
  \put(74, 17){\intab[1]{$0.47$ }}

  \put(90, 20){\intab[1]{$50\%$ }}
  \put(90, 17){\intab[1]{$0.65$ }}
\end{overpic}
\vspace*{-18pt}
\caption{
  Attack results under different sparsity levels.  
  \label{fig:dlg}}
\end{figure}\vspace*{-15pt}
\label{section:attack}

\section{The Number of Update Steps}
We explain below the reason that the optimal number of update steps, $t^{(k)}$,  must not be too large in \eqref{eq:t_optim}.
The weight $\w[k,t]$ is generated by the NTK method in \eqref{eq:w_evolution}. 
Let $\tw[k,t]$ denote the weight generated by the gradient descent method. 
For the two-layer neural network model given in \eqref{eq:relunet_func}, we will show the upper bound on the gap between $\w[k,t]$ and $\tw[k,t]$. 
To this end, we first give \paperlemma{lemma:bound_si} as a prerequisite. 
Define $\sset_i^{(k)}$ as the set of indices corresponding to neurons whose activation patterns are different from the broadcast version $\vecv[k]$ for an input $\x_i$:
\begin{equation}
    \sset_i^{(k)} \triangleq \left\{r\in [n] \ \big|\ \exists\; \vecv[k], \|\vecv[k] - \vecv[k]_r\|_2 \leqslant R, \, \indi[k]_{ir} \neq \ind[\langle \vecv^{(k)}, \x_i\rangle \geqslant 0] \right\}. 
\end{equation}
We upper bound the cardinality of $\sset_i^{(k)}$ in \paperlemma{lemma:bound_si}. 

\begin{Lemma}\label{lemma:bound_si}
    Under Assumptions \ref{assumption:initialization} to \ref{assumption:input}, with probability at least $1-\delta$, we have 
    \begin{equation}
        |\sset_i^{(k)}| \leqslant  \sqrt{\frac{2}{\pi} } \frac{n R }{\delta \alpha_{k}}, \quad \forall\; i \in [N_k].  
    \end{equation}
\end{Lemma}
\begin{proof}
    To bound $|\sset_i^{(k)}| = \sum\limits_{r=1}^n \ind[r \in \sset_i^{(k)}]$, consider an event $A_{ir}^{(k)}$ defined as follows: 
\begin{equation}\label{eq:event_air}
    A_{ir}^{(k)} \triangleq \left\{\exists\; \vecv[k], \|\vecv[k] - \vecv[k]_r\|_2 \leqslant R, \, \indi[k]_{ir} \neq \ind[\langle \vecv^{(k)}, \x_i \rangle \geqslant 0] \right\}. 
\end{equation}
Clearly, $\ind[r \in \sset_i^{(k)}] = \ind[A_{ir}^{(k)}]$. 
According to \paperassumption{assumption:input}, it can be shown that the event $A_{ir}^{(k)}$ happens if and only if $| (\vecv[k]_r)^\top \x_i | \leqslant R$ based on a geometric argument. 
From \paperassumption{assumption:initialization}, we have $ (\vecv[k]_r)^\top \x_i \sim \mathcal{N}(0, \x_i^\top \Sigma_k \x_i )$. 
The probability of event $A_{ir}^{(k)}$ is 
\begin{subequations}
\begin{align}
    \prob[A_{ir}^{(k)}] & = \prob\left[| (\vecv[k]_r)^\top \x_i| \leqslant R \right] \\
    & = \erf\left(\frac{R}{\sqrt{2 \x_i^\top \Sigma_k \x_i} }\right) \leqslant \sqrt{\frac{2}{\pi} } \frac{R}{\alpha_{k}},  \label{eq:prob_air_bound}
\end{align}
\end{subequations}
where the error function is given by $\erf(z) = \frac{2}{\sqrt{\pi}} \int_{0}^{z} e^{-t^{2}} \diff t$.
By Markov's inequality, we have with probability at least $1-\delta$, 
\begin{equation}\label{eq:interm_bound}
    \sum_{r=1}^n \ind[r \in \sset_i^{(k)}] \leqslant  \sqrt{\frac{2}{\pi} } \frac{n R }{\delta \alpha_{k}}. 
\end{equation}
The proof is complete.
\end{proof}\vspace*{-1em}

\begin{Lemma}\label{lemma:gap_ntk_gd}
Consider the residual term $\vecr[k,u]=\y^{(k)} - f^{(k,u)}(\X^{(k)})$. 
Suppose $\forall i \in [N_k]$, $|r^{(k,u)}_i|<\gamma^{(k,u)}$, where $\gamma^{(k,u)}$ is a positive constant and $\lim_{u \rightarrow \infty} \gamma^{(k,u)} = 0$. 
With probability at least $1-\delta$, we have
\begin{equation}
    \|\w[k, t] - \tw[k, t] \|_1 \leqslant \frac{\sqrt{2 n d_1 }  \eta }{\sqrt{\pi} \delta \alpha_{k}} \sum_{u=1}^{t-1} \gamma^{(k, u)}. 
\end{equation} 
\end{Lemma}
\begin{proof}
The weights $\w[k,t]$ and $\tw[k,t]$ can be written as 
\begin{align}
\w[k, t] = \frac{\eta}{N_k} \sum_{u=0}^{t-1} \nabla f^{(k,0)}(\X^{(k)}) \left[\y^{(k)} - f^{(k,u)}(\X^{(k)})\right]  + \w[k,0],  \\
\tw[k, t] = \frac{\eta}{N_k} \sum_{u=0}^{t-1} \nabla f^{(k,u)}(\X^{(k)}) \left[\y^{(k)} - f^{(k,u)}(\X^{(k)})\right] + \w[k,0]. 
\end{align}
The $\ell_1$ norm of the difference between $\w[k,t]$ and $\tw[k,t]$ is
\begin{subequations}
\begin{align}
    \|\w^{(k, t)} - \tw^{(k, t)} \|_1 &= \left\| \frac{\eta}{N_k} \sum_{u=1}^{t-1} \left[\nabla f^{(k,u)}(\X^{(k)}) -  \nabla f^{(k,0)}(\X^{(k)}) \right] \left[\mathbf{y}^{(k)} - f^{(k,u)}(\X^{(k)})\right] \right\|_1 \\
    & \leqslant \frac{\eta}{N_k} \sum_{u=1}^{t-1} \gamma^{(k,u)} \sum_{i=1}^{N_k} \left\|\frac{1}{\sqrt{n}} \sum_{r=1}^n c_r \x_i (\ind_{ir}^{(k,t)} - \ind_{ir}^{(k,0)}) \right\|_1 \\
    & \leqslant \frac{\eta}{N_k} \sum_{u=1}^{t-1} \gamma^{(k,u)} \sum_{i=1}^{N_k} \sum_{r=1}^{n} \frac{\sqrt{d_1}}{\sqrt{n}} |\ind_{ir}^{(k,t)} - \ind_{ir}^{(k,0)}| \\
    & = \frac{\eta}{N_k} \sum_{u=1}^{t-1} \gamma^{(k,u)} \sum_{i=1}^{N_k} \sum_{r=1}^{n} \frac{\sqrt{d_1}}{\sqrt{n}} \ind(r \in \sset^{(k)}_i) \\
    & \leqslant \frac{\sqrt{2n d_1 }  \eta }{\sqrt{\pi} \delta \alpha_k} \sum_{u=1}^{t-1} \gamma^{(k,u)}. \label{eq:upper_bound}
\end{align} 
\end{subequations} 
\end{proof}

The theoretical result indicates that the difference between the NTK weight $\w[k,t]$ and gradient descent weight $\tw[k,t]$ increases with the number of update steps $t$. 
We validate our theoretical result using real experiments and report the results in \paperfig{fig:w_diff}. 
The consistent trend between the empirical weight difference and the analytical upper bound confirms our analysis that increasing the number of update steps enlarges the gap between $\w[k,t]$ and $\tw[k,t]$.
In the NTK evolution scheme, one cannot choose an arbitrarily large $t$ as the final number of update step $t^{(k)}$. 

\begin{figure}
    \centering
    \subcaptionbox{}[0.45\textwidth]{
        \includegraphics[width=\linewidth]{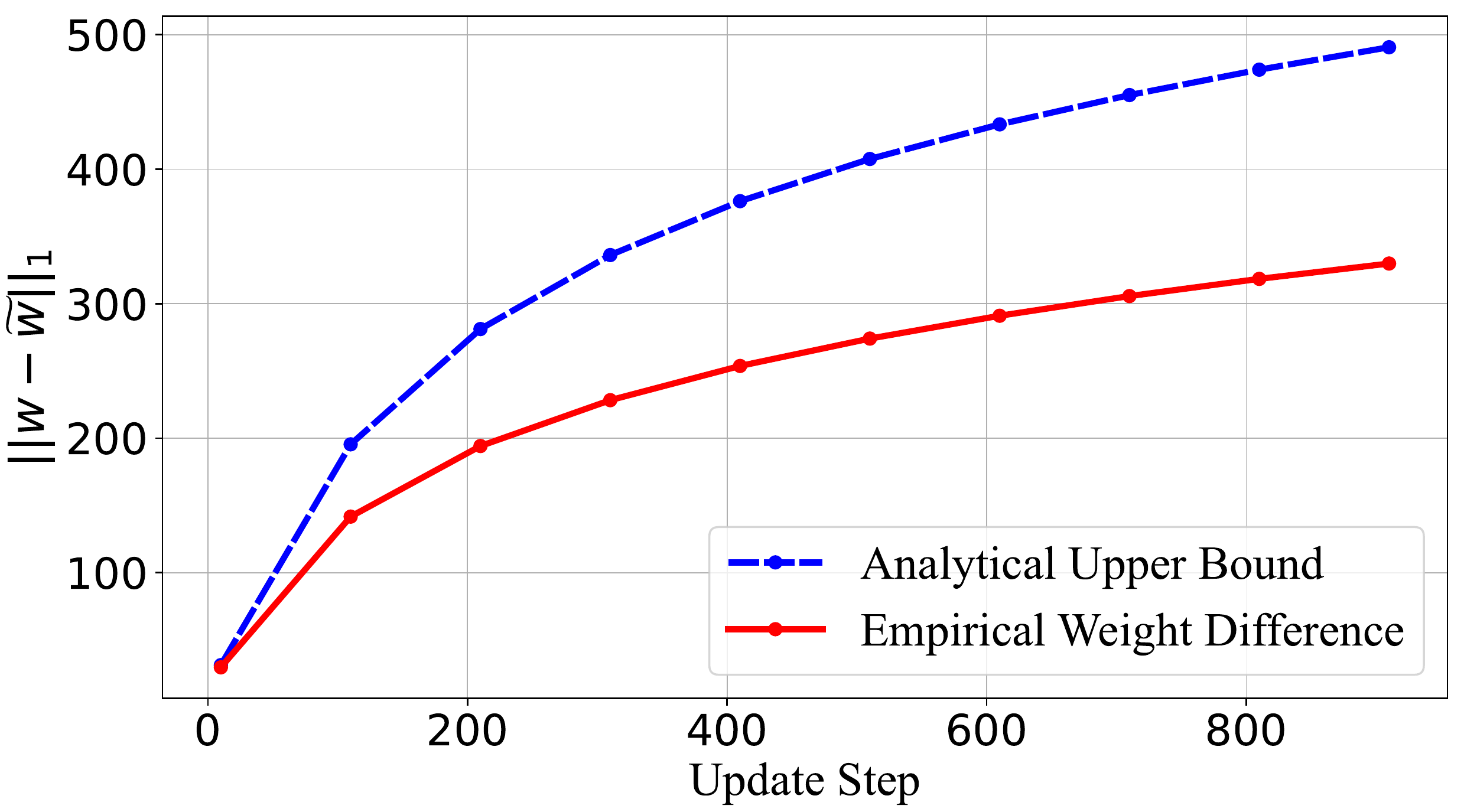}
    }
    \subcaptionbox{}[0.45\textwidth]{
        \includegraphics[width=\linewidth]{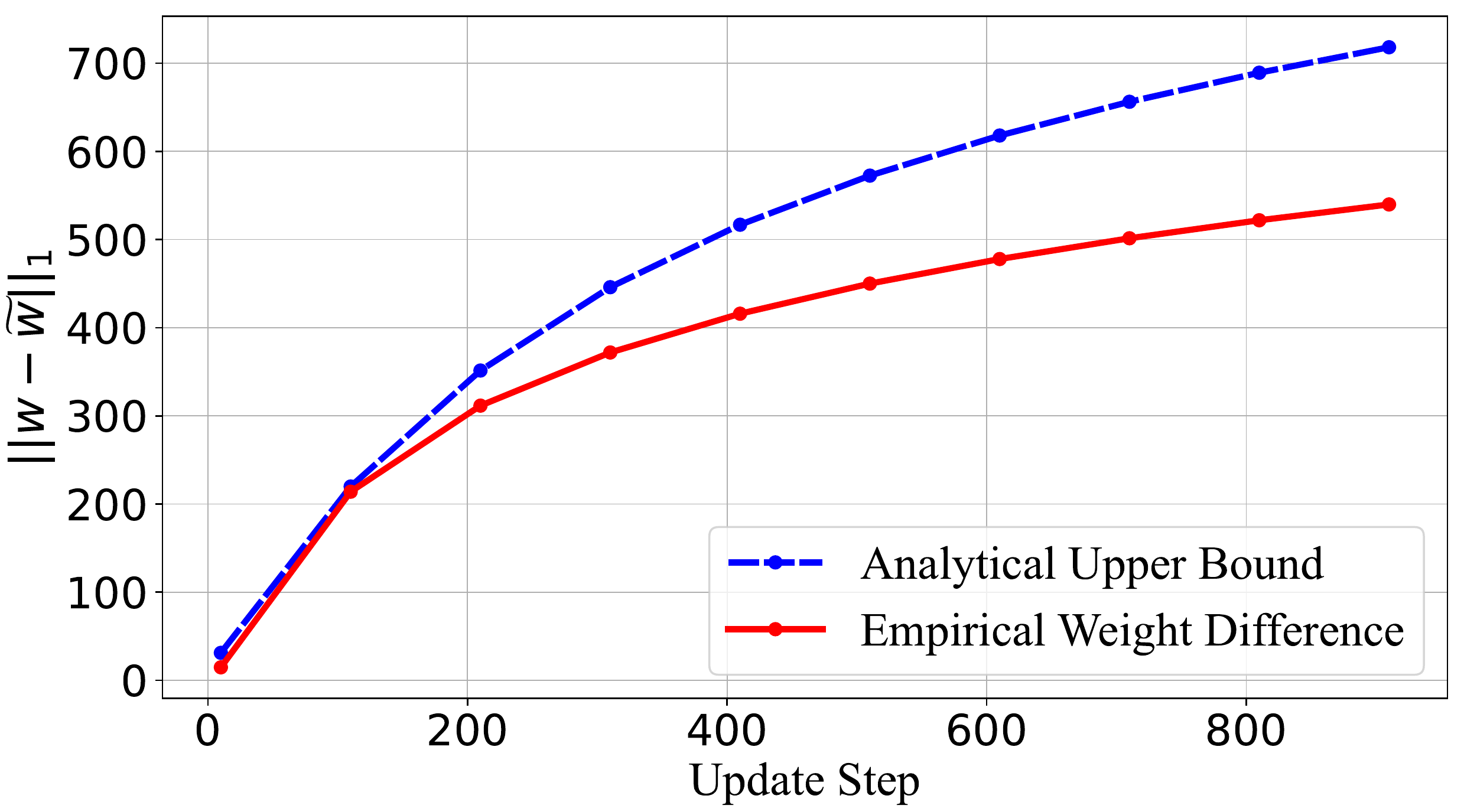}
    }
    \caption{The $\ell_1$ norm of the difference between the NTK weight $\w[k,t]$ and gradient descent weight $\tw[k,t]$ on the (a)~Fashion-MNIST dataset and (b)~FEMNIST dataset. 
    For the analytical upper bound, we calculate the term  $\sum_{u} \gamma^{(k,u)}$ in \eqref{eq:upper_bound} and set the coefficient to $2$.
    Both the theoretical result and real experiments show that the weight difference increases when more update steps are used.
    \label{fig:w_diff}} 
\end{figure}
\label{section:update_step}

\section{Proof of \papertheorem{th:ntkfl}}\label{section:proof_ntkfl}
We first present some lemmas to facilitate the convergence analysis.  
We bound the perturbation of the kernel matrix $\kernel[k,t]$ in \paperlemma{lemma:bound_kernel_time}. 
\begin{Lemma}\label{lemma:bound_kernel_time}
    Under Assumptions \ref{assumption:initialization} to \ref{assumption:input}, if $\forall\; r \in [n]$, $\|\vecv[k,t]_{r} - \vecv[k]_{r} \|_2 \leqslant R$, 
    then 
    \begin{equation}
        \| \kernel[k,t] - \kernel[k] \|_2 \leqslant \frac{2 \sqrt{2} N_k R}{\sqrt{\pi} \delta \alpha_k}. 
    \end{equation} 
\end{Lemma}
\begin{proof}
    We have 
\begin{subequations}
\begin{align}
    \| \kernel[k,t] - \kernel[k] \|^2_2 & \leqslant \| \kernel[k,t] - \kernel[k] \|^2_{\textrm{F}} \\
    & = \sum_{i=1}^{N_k}\sum_{j=1}^{N_k} \left[ (\kernel[k,t])_{ij} - (\kernel[k])_{ij} \right]^2 \\
    & = \frac{1}{n^2} \sum_{i=1}^{N_k}\sum_{j=1}^{N_k} (\x_i^\top \x_j)^2 \left( \sum_{r=1}^n \indi[k,t]_{ir} \indi[k,t]_{jr} -  \indi[k]_{ir} \indi[k]_{jr} \right)^2. \label{eq:kernel_perturb_bound}
\end{align}
\end{subequations}
Consider the event $A_{ir}$ defined in \eqref{eq:event_air}. 
Let $\phi^{(k,t)}_{ijr} \triangleq  \indi[k,t]_{ir} \indi[k,t]_{jr} -  \indi[k]_{ir} \indi[k]_{jr}$. 
If $\neg A_{ir}$ and $\neg A_{jr}$ happen, clearly we have $|\phi^{(k,t)}_{ijr}| = 0$. Therefore, the expectation of $|\phi^{(k,t)}_{ijr}|$ can be bounded as 
\begin{subequations}
\begin{align}
    \expect[\left|\phi^{(k,t)}_{ijr}\right|] & 
    \leqslant \prob( A_{ir} \cup A_{jr}) \cdot 1  \\
    & \leqslant \prob(A_{ir}) + \prob(A_{jr})  \\ 
    & \overset{\cirone}{\leqslant} 2 \sqrt{\frac{2}{\pi} } \frac{R }{\alpha_k}, 
\end{align}
\end{subequations} 
where $\cirone$ comes from \eqref{eq:prob_air_bound}. 
By Markov's inequality, we have with probability at least $1-\delta$, 
\begin{equation}\label{eq:bound_phi}
    |\phi^{(k,t)}_{ijr}| \leqslant 2 \sqrt{\frac{2}{\pi} } \frac{R }{\delta \alpha_k}. 
\end{equation}
Plugging \eqref{eq:bound_phi} into \eqref{eq:kernel_perturb_bound} yields 
\begin{equation}
    \| \kernel[k,t] - \kernel[k] \|^2_2 
    \leqslant \frac{ N_k^2}{n^2} \frac{8 n^2 R^2}{\pi \delta^2 \alpha_k^2} = \frac{8 N_k^2 R^2 }{\pi \delta^2 \alpha_k^2}. 
\end{equation}
Taking the square root on both sides completes the proof. 
\end{proof}

\begin{Lemma}\label{lemma:bound_kernel_var}
With probability at least $1 - \delta$, 
\begin{equation}\label{eq:bound_kernel_var}
    \| \kernel[k] - \kernelinf \|_2 \leqslant N_k  \sqrt{\frac{\ln \left(2 N_k^2 / \delta\right)}{2 n}}. 
\end{equation}
\end{Lemma}
\begin{proof}
We have 
\begin{equation}
    \| \kernel[k] - \kernelinf \|^2_{2} 
    \leqslant \| \kernel[k] - \kernelinf  \|^2_{\textrm{F}} 
    = \sum_{i=1}^{N_k}\sum_{j=1}^{N_k} \left[ (\kernel[k])_{ij} - (\kernelinf)_{ij} \right]^2.  
\end{equation}
Note that $(\kernel[k])_{ij} = \frac{1}{n}  \x_i^\top \x_j \sum\limits_{r=1}^n \indi[k]_{ir} \indi[k]_{jr} $, 
$(\kernel[k])_{ij} \in [-1, 1]$. 
By Hoeffding's inequality, we have with probability at least $1-\delta/N_k^2$,
\begin{equation}
    \left| (\kernel[k])_{ij} - (\kernelinf)_{ij} \right| \leqslant \sqrt{\frac{\ln \left(2 N_k^2 / \delta\right)}{2 n}}. 
\end{equation}
Applying the union bound over $i, j \in [N_k]$ yields 
\begin{equation}
    \| \kernel[k] - \kernelinf \|_{2} \leqslant N_k  \sqrt{\frac{\ln \left(2 N_k^2 / \delta\right)}{2 n}}. 
\end{equation}
The proof is complete. 
\end{proof}


Now we are going to prove \papertheorem{th:ntkfl}. 
\begin{Theorem}
    For the NTK-FL scheme under Assumptions \ref{assumption:initialization} to \ref{assumption:pdkernel}, 
    let the learning rate $\eta=\order{\frac{\lambda_{k}}{N_k}}$ and the neural network width $n = \Omega\left(\frac{ N_k^{2}}{\lambda_{k}^{2}} \ln \frac{ N_k^{2}}{\delta} \right)$, 
    then with probability at least $1-\delta$,
    \begin{equation}
        \normsq{f^{(k+1)}(\X^{(k)})-\y^{(k)}}  \leqslant   \left(1\!-\!\frac{\eta \lambda_{k} }{2 N_k} \right)^{t^{(k)}} \! \normsq{f^{(k)}(\X^{(k)}) - \y^{(k)}}, 
    \end{equation}
    where $t^{(k)}$ is the number of update steps defined in \eqref{eq:t_optim}.
\end{Theorem}

\begin{proof}\label{proof:ntkfl}
Taking the difference between successive neural network predictions yields 
\begin{equation}
f^{(k, t+1)}(\x_i) -  f^{(k, t)}(\x_i) \\ 
= \frac{1}{\sqrt{n}} \sum_{r=1}^{n} \left[ c_{r} \sigma\left((\vecv[k,t+1]_{r})^{\top} \mathbf{x}_{i}\right) -c_{r} \sigma\left((\vecv[k,t]_{r})^{\top} \mathbf{x}_{i}\right) \right]. \\
\end{equation}
We decompose the difference term to the sum of $d^{\RN{1}}_i$ and $d^{\RN{2}}_i$, based on the set $\sset_i^{(k)}$:  
\begin{subequations}
\begin{align}
    d_i^{\RN{1}} &\triangleq \frac{1}{\sqrt{n}} \sum_{r \notin \sset_i^{(k)}} \left[ c_{r} \sigma \left( (\vecv[k,t+1]_{r})^\top \x_i \right) - c_r \sigma \left( (\vecv[k,t]_{r})^\top \x_i \right) \right], \\
    d_i^{\RN{2}} &\triangleq \frac{1}{\sqrt{n}} \sum_{r \in \sset_i^{(k)}} \left[ c_{r} \sigma \left( (\vecv[k,t+1]_{r})^\top \x_i \right) - c_{r} \sigma \left( (\vecv[k,t]_{r})^\top \x_i \right) \right].
\end{align}
\end{subequations}
Consider the residual term
\begin{subequations}
\begin{align}
    & \normsq{f^{(k,t+ 1)}(\X^{(k)}) - \y^{(k)}} \\
    = &\;  \normsq{f^{(k,t+1)}(\X^{(k)}) - f^{(k,t)}(\X^{(k)}) + f^{(k,t)}(\X^{(k)})  - \y^{(k)}}  \\ 
    = &\;  \normsq{f^{(k,t)}(\X^{(k)}) - \y^{(k)}} + 2 \left \langle \vecd^{\RN{1}} + \vecd^{\RN{2}},  f^{(k,t)}(\X^{(k)})  - \y^{(k)} \right \rangle + \normsq{f^{(k,t+1)}(\X^{(k)}) - f^{(k,t)}(\X^{(k)})}. 
\end{align}
\end{subequations}
We will give upper bounds for the inner product terms $\left \langle \vecd^{\RN{1}}, f^{(k,t)}(\X^{(k)})  - \y^{(k)} \right \rangle$, $\left \langle \vecd^{\RN{2}}, f^{(k,t)}(\X^{(k)})  - \y^{(k)} \right \rangle$, and the difference term  $\normsq{f^{(k,t+1)}(\X^{(k)}) - f^{(k,t)}(\X^{(k)})}$, separately. 
Based on the property of the set $\sset_i^{(k)}$, we have 
\begin{subequations}
\begin{align}
    d_{i}^{\RN{1}} &=-\frac{\eta}{\sqrt{n}} \sum_{r \notin S_{i}} c_r \left \langle\nabla_{\vecv_{r}} L, \mathbf{x}_{i}\right \rangle \indi[k,t]_{i r} \\
    & = -\frac{\eta}{n N_k} \sum_{j=1}^{N_k} \left( f^{(k,t)}(\x_j) - y_{j} \right) \x_j^\top \x_i \sum_{r \notin S_{i}} c_r^2 \, \indi[k,t]_{i r} \indi[k,t]_{j r} \\
    & = -\frac{\eta}{N_k} \sum_{j=1}^{N_k} \left( f^{(k,t)}(\x_j) - y_{j} \right) \left( (\kernel[k,t])_{ij} - (\pkernel[k,t])_{ij} \right), 
\end{align}
\end{subequations}
where  $(\pkernel[k,t])_{ij}$ is defined as 
\begin{equation}
    (\pkernel[k,t])_{ij} \triangleq \frac{1}{n} \x_i^\top \x_j \sum\limits_{r \in \sset_i^{(k)}}^{n}  \indi[k,t]_{ir} \indi[k,t]_{jr}. 
\end{equation}
For the inner product term $\left \langle \vecd^{\RN{1}}, f^{(k,t)}(\X^{(k)})  - \y^{(k)} \right \rangle$, we have 
\begin{equation}
    \left \langle \vecd^{\RN{1}}, f^{(k,t)}(\X^{(k)}) - \y^{(k)} \right \rangle 
    = - \frac{\eta}{N_k} ( f^{(k,t)}(\X^{(k)}) - \y^{(k)} )^\top (\kernel[k,t] - \pkernel[k,t]) (f^{(k,t)}(\X^{(k)}) - \y^{(k)} ). 
\end{equation}
Let $T_1$ and $T_2$ denote the following terms 
\begin{subequations}
\begin{align}
    T_1 & \triangleq -( f^{(k,t)}(\X^{(k)}) - \y^{(k)} )^\top \kernel[k,t] (f^{(k,t)}(\X^{(k)}) - \y^{(k)} ), \\
    T_2 & \triangleq ( f^{(k,t)}(\X^{(k)}) - \y^{(k)} )^\top \pkernel[k,t] (f^{(k,t)}(\X^{(k)}) - \y^{(k)} ). 
\end{align}
\end{subequations}
With probability at least $1-\delta$, $T_1$ can be bounded as:
\begin{subequations}
\begin{align}
    T_1 & = -( f^{(k,t)}(\X^{(k)}) - \y^{(k)} )^\top (\kernel[k,t] - \kernel[k] + \kernel[k] - \kernelinf + \kernelinf ) ( f^{(k,t)}(\X^{(k)}) - \y^{(k)} ) \\ 
    & \leqslant -( f^{(k,t)}(\X^{(k)}) - \y^{(k)} )^\top(\kernel[k,t] - \kernel[k])( f^{(k,t)}(\X^{(k)}) - \y^{(k)} ) \nonumber \\
    & \hspace{9pt} -( f^{(k,t)}(\X^{(k)}) - \y^{(k)} )^\top(\kernel[k] - \kernelinf)( f^{(k,t)}(\X^{(k)}) - \y^{(k)} )
     - \lambda_{k} \normsq{f^{(k,t)}(\X^{(k)}) - \y^{(k)}}  \\ 
    & \overset{\cirone}{\leqslant}  \left( \frac{2 \sqrt{2} N_k R}{\sqrt{\pi} \delta \alpha_{k}} + N_k \sqrt{\frac{\ln \left(2 N_k^2 / \delta\right)}{2 n}} - \lambda_{k} \right) \normsq{f^{(k,t)}(\X^{(k)}) - \y^{(k)}}, \label{eq:T1}
\end{align}
\end{subequations}
where $\cirone$ comes from \paperlemma{lemma:bound_kernel_time} and \paperlemma{lemma:bound_kernel_var}. 
To bound the term $T_2$, consider the $\ell_{2}$ norm of the matrix $\pkernel[k,t]$. 
With probability at least $1-\delta$, we have:
\begin{subequations}
\begin{align}
    \| \pkernel[k,t] \|_2 & \leqslant \| \pkernel[k,t] \|_{\textrm{F}} \\
    & = \left(\sum_{i=1}^{N_k} \sum_{j=1}^{N_k}  \left(\frac{1}{n} \sum_{r\in \sset_i^{(k)}}  \x^\top_{i} \x_{j}  \indi[k,t]_{ir} \indi[k,t]_{jr} \right)^2 \right)^{\frac{1}{2}} \\
    & \leqslant \frac{N_k}{n} |\sset_i^{(k)}| 
    \overset{\cirone}{\leqslant}  \sqrt{\frac{2}{\pi} } \frac{N_k R }{\delta \alpha_{k}}, 
\end{align}
\end{subequations}
where $\cirone$ comes from \paperlemma{lemma:bound_si}. 
Therefore, with probability at least $1-\delta$, we have 
\begin{equation}
    T_2 \leqslant \sqrt{\frac{2}{\pi} } \frac{N_k R }{\delta \alpha_{k}} \normsq{f^{(k,t)}(\X^{(k)}) - \y^{(k)} }.  \label{eq:T2}
\end{equation}
Combine the results of \eqref{eq:T1} and \eqref{eq:T2}: 
\begin{equation}
    \left \langle \vecd^{\RN{1}}, f^{(k,t)}(\X^{(k)}) - \y^{(k)} \right \rangle  \leqslant
    \eta \left( \frac{3 \sqrt{2} R}{ \sqrt{\pi} \delta \alpha_{k}} + \sqrt{\frac{\ln \left(2 N_k^2 / \delta\right)}{2 n}} - \frac{\lambda_{k}}{N_k} \right)   \normsq{f^{(k,t)}(\X^{(k)}) - \y^{(k)} }. \label{eq:bound_inner1}
\end{equation}
For the inner product term $\left \langle \vecd^{\RN{2}}, f^{(k,t)}(\X^{(k)})  - \y^{(k)} \right \rangle$, we first bound $\|\vecd^{\RN{2}}\|_2^2$ as follows:
\begin{subequations}
\begin{align}
    \|\vecd^{\RN{2}}\|_2^2 
    & =  \sum_{i=1}^{N_k} \left( \frac{1}{\sqrt{n}} \sum_{r \in \sset_i^{(k)}} \left[ c_{r} \sigma \left( (\vecv[k,t+1]_{r})^\top \x_i \right) - c_{r} \sigma \left( (\vecv[k,t]_{r})^\top \x_i \right) \right] \right)^2 \\
    & \overset{\cirone}{\leqslant} \frac{\eta^2}{n} \sum_{i=1}^{N_k} |\sset_i^{(k)}| \sum_{r \in \sset_i^{(k)}}  ( c_r \langle \nabla_{\vecv_r} L, \x_i \rangle )^2  \\
    & \overset{\cirtwo}{\leqslant} \frac{\eta^2}{n} \sum_{i=1}^{N_k} |\sset_i^{(k)}| \sum_{r \in \sset_i^{(k)}}  \|\nabla_{\vecv_r} L\|_2^2 \, \|\x_i\|_2^2 \\
    & \leqslant \frac{\eta^2 N_k}{n} |\sset_i^{(k)}|^2 \max_{r\in[n]} \normsq{\nabla_{\vecv_r} L } \\
    & \leqslant \frac{\eta^2 |\sset_i^{(k)}|^2}{n^2 } \normsq{f^{(k,t)}(\X^{(k)}) - \y^{(k)} } \label{eq:bound_d2},
\end{align}
\end{subequations}
where $\cirone$ comes from the Lipschitz continuity of the ReLU function $\sigma(\cdot)$, $\cirtwo$ holds due to Cauchy--Schwartz inequality. 
Plug \eqref{eq:interm_bound} into \eqref{eq:bound_d2}, we have with probability at least $1-\delta$:
\begin{equation}
    \|\vecd^{\RN{2}}\|_2^2 \leqslant \frac{2 \eta^2 R^2 }{\pi \delta^2 \alpha_{k}^2} \normsq{f^{(k,t)}(\X^{(k)}) - \y^{(k)} }. 
\end{equation}
The inner product term $\left \langle \vecd^{\RN{2}}, f^{(k,t)}(\X^{(k)})  - \y^{(k)} \right \rangle$ can be bounded as 
\begin{equation}
    \left \langle \vecd^{\RN{2}}, f^{(k,t)}(\X^{(k)})  - \y^{(k)} \right \rangle \leqslant \frac{\sqrt{2} \eta R}{\sqrt{\pi} \delta \alpha_{k}} \normsq{f^{(k,t)}(\X^{(k)}) - \y^{(k)}}.  \label{eq:bound_inner2}
\end{equation}
Finally, the bound for the difference term is derived as 
\begin{equation}
    \normsq{f^{(k,t+1)}(\X^{(k)}) - f^{(k,t)}(\X^{(k)})} 
    \leqslant \sum_{i=1}^{N_k} \left( \frac{\eta}{\sqrt{n}} \sum_{r=1}^n c_r  \langle\nabla_{\vecv_{r}} L, \mathbf{x}_{i} \rangle  \right)^2 
    \leqslant \eta^2 \normsq{f^{(k,t)}(\X^{(k)}) - \y^{(k)} }. \label{eq:bound_diff_term}
\end{equation}
Combine the results of \eqref{eq:bound_inner1}, \eqref{eq:bound_inner2} and \eqref{eq:bound_diff_term}:
\begin{equation}
    \normsq{f^{(k,t+ 1)}(\X^{(k)}) - \y^{(k)}} \leqslant \left[ 1 + \frac{8 \sqrt{2} \eta R}{\sqrt{\pi} \delta \alpha_{k}} + 2\eta \sqrt{\frac{\ln \left(2 N_k^2 / \delta\right)}{2 n}} -  \frac{2 \eta \lambda_{k}}{N_k} + \eta^2\right] \normsq{f^{(k,t)}(\X^{(k)}) - \y^{(k)}}. 
\end{equation}
Let $R = O\left(\frac{\delta \alpha_{k} \lambda_{k}}{N_k}\right)$, $n = \Omega\left(\frac{ N_k^{2}}{\lambda_{k}^{2}} \ln \frac{ N_k^{2}}{\delta} \right)$, and $\eta = O( \frac{\lambda_{k}}{N_k})$, we have 
\begin{equation}\label{eq:recursive}
    \normsq{f^{(k,t+ 1)}(\X^{(k)}) - \y^{(k)}} \leqslant \left( 1 - \frac{\eta  \lambda_{k}}{2 N_k} \right) \normsq{f^{(k,t)}(\X^{(k)}) - \y^{(k)}}. 
\end{equation}
Invoking the inequality \eqref{eq:recursive} recursively completes the proof.
\end{proof}

\section{Proof of \papertheorem{th:fedavg}}\label{section:proof_fedavg}
\begin{Theorem}
    For FedAvg under Assumptions \ref{assumption:initialization} to \ref{assumption:pdkernel}, 
    let the learning rate $\eta=\order{\frac{\lambda_{k}}{\tau N_k M_k}}$ and 
    the neural network width $n = \Omega\left(\frac{ N_k^{2}}{\lambda_{k}^{2}} \ln \frac{ N_k^{2}}{\delta} \right)$, 
    then with probability at least $1-\delta$, 
    \begin{equation}
        \normsq{f^{(k+1)}(\X^{(k)}) - \y^{(k)}} \leqslant \left( 1 - \frac{\eta \tau \lambda_{k}}{2 N_k M_k}  \right) \normsq{f^{(k)}(\X^{(k)}) - \y^{(k)}}, 
    \end{equation}
    where $\tau$ is the number of local iterations, and $M_k$ is the number of clients in communication round $k$. 
\end{Theorem}
\begin{proof}
We first construct a different set of kernel matrices $\{\asyker[k], \asyker[k,\tau]_m\}$ similar to \cite{huang2021fl}. 
Let $\indi[k,u]_{imr} \triangleq \ind[] \{ \langle \vecv[k,u]_{m, r}, \x_{i} \rangle \geqslant 0 \}$, the $(i,j)$th entry of  $\asyker[k,u]_{m}$ and $\asyker[k, u]$ is defined as 
\begin{subequations}
\begin{align}
    (\asyker[k, u]_{m})_{ij} & \triangleq \frac{1}{n} \x^\top_i \x_j \sum_{r=1}^{n} \indi[k,0]_{imr} \indi[k,u]_{jmr}, \\
    (\asyker[k, u])_{ij} & \triangleq (\asyker[k, u]_m)_{ij},\quad \textrm{ if } (\x_{j}, y_j) \in \mathcal{D}_{m}. 
\end{align}
\end{subequations}
Taking the difference between successive terms yields 
\begin{equation}
f^{(k+1)}(\x_i) -  f^{(k)}(\x_i) \\ 
= \frac{1}{\sqrt{n}} \sum_{r=1}^{n} \left[ c_{r} \sigma\left((\vecv[k+1]_{r})^{\top} \mathbf{x}_{i}\right) - c_{r} \sigma\left((\vecv[k]_{r})^{\top} \mathbf{x}_{i}\right) \right]. \\
\end{equation}
We decompose the difference term to the sum of $d^{\RN{1}}_i$ and $d^{\RN{2}}_i$, based on the set $\sset^{(k)}_i$ and its complement: 
\begin{subequations}
\begin{align}
    d_i^{\RN{1}} &\triangleq \frac{1}{\sqrt{n}} \sum_{r \notin \sset^{(k)}_i} \left[ c_{r} \sigma \left( (\vecv[k+1]_{r})^\top \x_i \right) - c_r \sigma \left( (\vecv[k]_{r})^\top \x_i \right) \right], \\
    d_i^{\RN{2}} &\triangleq \frac{1}{\sqrt{n}} \sum_{r \in \sset^{(k)}_i} \left[ c_{r} \sigma \left( (\vecv[k+1]_{r})^\top \x_i \right) - c_{r} \sigma \left( (\vecv[k]_{r})^\top \x_i \right) \right].
\end{align}
\end{subequations}
Consider the residual term
\begin{subequations}
\begin{align}
    & \normsq{f^{(k + 1)}(\X^{(k)}) - \mathbf{y}} \\
    = &\;  \normsq{f^{(k+1)}(\X^{(k)}) - f^{(k)}(\X^{(k)}) + f^{(k)}(\X^{(k)})  - \mathbf{y}}  \\ 
    = &\;  \normsq{f^{(k)}(\X^{(k)}) - \y^{(k)}} + 2 \left \langle \vecd^{\RN{1}} + \vecd^{\RN{2}},  f^{(k)}(\X^{(k)})  - \mathbf{y} \right \rangle + \normsq{f^{(k+1)}(\X^{(k)}) - f^{(k)}(\X^{(k)})}. 
\end{align}
\end{subequations}
We will give upper bounds for the inner product terms $\left \langle \vecd^{\RN{1}}, f^{(k)}(\X^{(k)})  - \mathbf{y} \right \rangle$, 
$\left \langle \vecd^{\RN{2}}, f^{(k)}(\X^{(k)})  - \mathbf{y} \right \rangle$, 
and the difference term  $\normsq{f^{(k+1)}(\X^{(k)}) - f^{(k)}(\X^{(k)})}$, separately. 
For an input $\x \in \mathbb{R}^{d_1}$, let $f^{(k,u)}_m (\x) \triangleq \frac{1}{\sqrt{n}} \sum\limits_{r=1}^{n} c_r \sigma(\langle\vecv[k,u]_{m,r}), \x\rangle) $. 
By the update rule of FedAvg, the relation between the weight vector $\vecv[k]_r$ in successive communication rounds is: 
\begin{subequations}
\begin{align}
    \vecv[k+1]_{r} &= \vecv[k]_{r} - \frac{\eta}{M_k} \sum_{m\in \cset_{k}} \sum_{u=0}^{\tau-1} \nabla L_{\vecv[k,u]_r} \\ 
    & = \vecv[k]_{r} - \frac{\eta c_r}{N_{k} \sqrt{n} M_k } \sum_{m\in \cset_{k}} \sum_{u=0}^{\tau-1} \sum_{j \in \I_m}  (f^{(k,u)}_m(\x_{j}) - y_{j}) \x_{j} \indi[k,u]_{jmr}.  
\end{align}
\end{subequations}
Based on the property of the set $\sset^{(k)}_i$, we have 
\begin{subequations}
\begin{align}
    d_i^{\RN{1}} 
    & = -\frac{1 }{\sqrt{n} } \sum_{m \in \cset_{k}} \sum_{u=0}^{\tau-1} \sum_{r \notin \sset^{(k)}_{i}} c_r \left \langle \vecv[k+1]_r - \vecv[k]_{r}, \mathbf{x}_{i}\right \rangle \indi[k]_{ir} \\
    & = -\frac{\eta }{N_{k} n M_k} \sum_{m \in \cset_{k}} \sum_{u=0}^{\tau-1} \sum_{r \notin \sset^{(k)}_{i}} \sum_{j \in \I_m} (f^{(k,u)}_m(\x_{j}) - y_{j}) \x_i^\top \x_j \indi[k]_{ir} \indi[k,u]_{jmr} \\ 
    & = -\frac{\eta }{N_{k} M_k} \sum_{m \in \cset_{k}} \sum_{u=0}^{\tau-1} \sum_{j \in \I_m} (f^{(k,u)}_m(\x_{j}) - y_{j}) \left[ (\asyker[k,u]_{m})_{ij} -  (\pasyker[k,u]_{m})_{ij} \right]. 
\end{align}
\end{subequations}
For the inner product term $\left \langle \vecd^{\RN{1}}, f^{(k)}(\X^{(k)})  - \mathbf{y} \right \rangle$, we have 
\begin{equation}
    \left \langle \vecd^{\RN{1}}, f^{(k)}(\X^{(k)}) - \y^{(k)} \right \rangle 
    = -\frac{\eta }{N_{k} M_k} \sum_{u=0}^{\tau-1} (f^{(k)}(\X^{(k)}) - \y^{(k)})^\top (\asyker[k,u] -  \pasyker[k,u])   (f^{(k,u)}_m (\X^{(k)}) - \y^{(k)}). 
\end{equation}
Let $T_1$ and $T_2$ denote the following terms 
\begin{subequations}
\begin{align}
    T_1 & \triangleq -( f^{(k)}(\X^{(k)}) - \y^{(k)} )^\top \asyker[k,u] (f^{(k,u)}_{g}(\X^{(k)}) - \y^{(k)} ), \\
    T_2 & \triangleq ( f^{(k)}(\X^{(k)}) - \y^{(k)} )^\top \pasyker[k,u] (f^{(k,u)}_{g}(\X^{(k)}) - \y^{(k)} ), 
\end{align}
\end{subequations}
where $f^{(k,u)}_{g}(\X^{(k)}) \triangleq [f^{(k,u)}_1(\X_1)^\top, \cdots, f^{(k,u)}_{M_k}(\X_{M_k})^\top ]^\top$. 
We are going to bound $T_1$ and $T_2$ separately. 
$T_1$ can be written as: 
\begin{subequations}
    \begin{align}
        T_1 & = -( f^{(k)}(\X^{(k)}) - \y^{(k)} )^\top (\asyker[k,u] - \kernel[k] + \kernel[k] - \kernelinf + \kernelinf ) ( f_{g}^{(k,u)}(\X^{(k)}) - \y^{(k)} ) \\ 
        & = -( f^{(k)}(\X^{(k)}) - \y^{(k)} )^\top(\asyker[k,u] - \kernel[k])( f^{(k,u)}_{g}(\X^{(k)}) - \y^{(k)} ) \nonumber \\
        & \hspace{9pt} -( f^{(k)}(\X^{(k)}) - \y^{(k)} )^\top(\kernel[k] - \kernelinf)( f^{(k,u)}_{g}(\X^{(k)}) - \y^{(k)} ) \nonumber \\
        & \hspace{9pt} -( f^{(k)}(\X^{(k)}) - \y^{(k)} )^\top \kernelinf  ( f^{(k)}(\X^{(k)}) - \y^{(k)} )  \nonumber \\ 
        & \hspace{9pt} -( f^{(k)}(\X^{(k)}) - \y^{(k)} )^\top \kernelinf  ( f^{(k,u)}_{g}(\X^{(k)}) -  f^{(k)}(\X^{(k)}) ).  \label{eq:bound_T1_interm}
    \end{align}
\end{subequations}
First, we bound the norm of $f^{(k,u)}_{g}(\X^{(k)}) - \y^{(k)}$. 
It can be shown that
\begin{subequations}
\begin{align}
    \|f^{(k,u)}_m (\X_m) - \y_m \|_2 
    & = \|f^{(k,u)}_m (\X_m) - f^{(k,u-1)}_m (\X_m) + f^{(k,u-1)}_m (\X_m) - \y_m \|_2 \\
    & \leqslant \|f^{(k,u)}_m (\X_m)- f^{(k,u-1)}_m(\X_m) \|_2 + \| f^{(k,u-1)}_m (\X_m) - \y_m \|_2 \\
    & \overset{\cirone}{\leqslant} (1+\eta) \|f^{(k,u-1)}_m(\X_m) - \y_m \|_2, \label{eq:local_interm}
\end{align}
\end{subequations}
where $\cirone$ holds based on the derivation of \eqref{eq:bound_diff_term}. 
Applying \eqref{eq:local_interm} recursively yields 
\begin{equation}
    \|f^{(k,u)}_m (\X_m) - \y_m \|_2 \leqslant (1+\eta)^{u} \| f^{(k)}(\X_m) - \y_m \|_2.  \label{eq:local_diff}
\end{equation}
The bound for $\|f^{(k,u)}_{g}(\X^{(k)}) - \y^{(k)}\|^2_2$ can thus be derived as 
\begin{subequations}
\begin{align}
    \|f^{(k,u)}_{g}(\X^{(k)}) - \y^{(k)}\|^2_2 
    & = \sum_{i=1}^{N_{k}} \left[f^{(k,u)}_{g}(\x_i) - y_i\right]^2  \\
    & = \sum_{m \in \cset_{k}} \normsq{f^{(k,u)}_m (\X_m) - \y_m} \\
    & \leqslant (1+\eta)^{2u} \normsq{f^{(k)}(\X^{(k)}) - \y^{(k)} }. \label{eq:bound1}
\end{align}
\end{subequations}
Second, following the steps in \paperlemma{lemma:bound_kernel_time}, it can be shown that with probability at least $1-\delta$, 
\begin{equation}
    \| \asyker[k,t] - \kernel[k] \|_2 \leqslant \frac{2 \sqrt{2} N_{k} R}{\sqrt{\pi} \delta \alpha}. \label{eq:bound2}
\end{equation}
We also bound the difference between $f^{(k,u)}_{g}(\X^{(k)})$ and $f^{(k)}(\X^{(k)})$ as follows:  
\begin{subequations}
\begin{align}
    \| f^{(k,u)}_{g}(\X^{(k)}) -  f^{(k)}(\X^{(k)}) \|_2 
    & \overset{\cirone}{\leqslant} \sum_{v=1}^{u} \| f^{(k,v)}_{g}(\X^{(k)}) - f^{(k,v-1)}_{g}(\X^{(k)}) \|_2 \\
    & \overset{\cirtwo}{\leqslant} \sum_{v=1}^{u} \eta \| f^{(k,v-1)}_{g}(\X^{(k)}) - \y^{(k)} \|_2 \\
    & \overset{\cirthree}{\leqslant} \sum_{v=1}^{u} \eta (1+\eta)^{v-1} \| f^{(k)}(\X^{(k)}) - \y^{(k)} \|_2 \\
    & = \left[(1+\eta)^{u} -1 \right]\| f^{(k)}(\X^{(k)}) - \y^{(k)}  \|_2, \label{eq:bound3}
\end{align}
\end{subequations}
where $\cirone$ holds due to triangle inequality, $\cirtwo$ comes from \eqref{eq:bound_diff_term}, $\cirthree$ comes from \eqref{eq:bound1}. 
Plugging  the results from \eqref{eq:bound1}, \eqref{eq:bound2}, and \eqref{eq:bound3} into \eqref{eq:bound_T1_interm}, we have with probability at least $1-\delta$, 
\begin{equation}
    T_1 \leqslant \left[ (1+\eta)^{u} \left( \frac{2 \sqrt{2} N_{k} R}{\sqrt{\pi} \delta \alpha} +  N_{k} \sqrt{\frac{\ln \left(2 N_{k}^2 / \delta\right)}{2 n}} + \kappa \lambda_{k} \right)  
    - (1+ \kappa)\lambda_{k}  \right] \|f^{(k)}(\X^{(k)}) - \y^{(k)}\|_2^2, \label{eq:T1_}
\end{equation}
where $\kappa$ is the condition number of the matrix $\kernelinf$. 
Next, consider the bound for $T_2$. The $\ell_2$ norm of $\pasyker[k,u]$ can be bounded as 
\begin{subequations}
\begin{align}
    \| \pasyker[k,u] \|_2 & \leqslant \| \pasyker[k,u] \|_{\textrm{F}} \\
    & = \left(\sum_{i=1}^{N_{k}} \sum_{m \in \cset_k} \sum_{j \in \I_m}  \left(\frac{1}{n} \sum_{r\in \sset^{(k)}_i}  \x^\top_{i} \x_{j}  \indi[k]_{ir} \indi[k,u]_{jmr}  \right)^2 \right)^{\frac{1}{2}} \\
    & \leqslant \frac{N_{k}}{n} |\sset^{(k)}_i| 
    \overset{\cirone}{\leqslant}  \sqrt{\frac{2}{\pi} } \frac{N_{k} R }{\delta \alpha}, 
\end{align}
\end{subequations}
where $\cirone$ comes from \paperlemma{lemma:bound_si}. 
Therefore, we have with probability at least $1 - \delta$, 
\begin{equation}
    T_2 \leqslant (1+\eta)^u \sqrt{\frac{2}{\pi} } \frac{N_{k} R }{\delta \alpha} \normsq{f^{(k)}(\X^{(k)}) - \y^{(k)} }.  \label{eq:T2_}
\end{equation}
Combine the results of \eqref{eq:T1_} and \eqref{eq:T2_}:
\begin{equation}
\begin{aligned}
    \left \langle \vecd^{\RN{1}}, f^{(k)}(\X^{(k)}) - \y^{(k)} \right \rangle  \leqslant 
    &\; \frac{\tau}{M_k} \Bigg[ \left(\eta + \frac{(\tau-1)}{2}\eta^2 + o(\eta^2) \right) \left( \frac{3 \sqrt{2} R}{ \sqrt{\pi} \delta \alpha} + \sqrt{\frac{\ln \left(\frac{2 N_{k}^2}{\delta} \right)}{2 n}} + \frac{\kappa \lambda_{k}}{N_{k}} \right) \\
    & - \frac{(1+\kappa) \eta \lambda_{k}}{N_{k}}  \Bigg]  \normsq{f^{(k)}(\X^{(k)}) - \y^{(k)} }. \label{eq:bound_inner1_} 
\end{aligned}
\end{equation}
For the inner product term $\left \langle \vecd^{\RN{2}}, f^{(k)}(\X^{(k)})  - \mathbf{y} \right \rangle$, we first bound $\|\vecd^{\RN{2}}\|_2^2$ with probability at least $1-\delta$:
\begin{subequations}
\begin{align}
    \|\vecd^{\RN{2}}\|_2^2 
    & =  \sum_{i=1}^{N_{k}} \left( \frac{1}{\sqrt{n}} \sum_{r \in \sset^{(k)}_i} \left[ c_{r} \sigma \left( (\vecv[k+1]_{r})^\top \x_i \right) - c_{r} \sigma \left( (\vecv[k]_{r})^\top \x_i \right) \right] \right)^2 \\
    & \leqslant \frac{1}{n} \sum_{i=1}^{N_{k}} |\sset^{(k)}_i| \sum_{r \in \sset^{(k)}_i} \left( c_r \langle \vecv[k+1]_{r} - \vecv[k]_{r}, \x_i \rangle \right)^2 \\
    & \leqslant \frac{1}{n} \sum_{i=1}^{N_{k}} |\sset^{(k)}_i| \sum_{r \in \sset^{(k)}_i} \left( \frac{\eta c_r}{N_{k} \sqrt{n} M_k} \sum_{m \in \cset_k} \sum_{u=0}^{\tau-1}  \sum_{j \in \I_m} (f^{(k,u)}_{m}(\x_j) - y_j) \indi[k,u]_{jmr}  \right)^2 \\
    & \leqslant \frac{\eta^2}{N_{k}^2 n^2 M_k^2} \sum_{i=1}^{N_{k}} |\sset^{(k)}_i| \sum_{r \in \sset^{(k)}_i} \left( \sum_{m \in \cset_k} \sum_{u=0}^{\tau-1}  \sum_{j \in \I_m} \left| f^{(k,u)}_{m}(\x_j) - y_j \right| \right)^2 \\
    & \leqslant \frac{\eta^2}{N_{k}^2 n^2 M_k^2} \sum_{i=1}^{N_{k}} |\sset^{(k)}_i| \sum_{r \in \sset^{(k)}_i}  \left( \sum_{m \in \cset_k} \sum_{u=0}^{\tau-1} |\I_m| \left\| f^{(k,u)}_{m}(\X_m) - \y_m \right\|_2 \right)^2.
\end{align}
\end{subequations}
By apply the results from previous steps, we have 
\begin{subequations}
\begin{align}
    \|\vecd^{\RN{2}}\|_2^2 
    & \overset{\cirone}{\leqslant} \frac{\eta^2}{N_{k}^2 n^2 M_k^2} \sum_{i=1}^{N_{k}} |\sset^{(k)}_i| \sum_{r \in \sset^{(k)}_i} \left( \sum_{m \in \cset_k} \sum_{u=0}^{\tau-1} (1+\eta)^u |\I_m| \left\| f^{(k)}(\X_m) - \y_m \right\|_2 \right)^2 \\
    & \overset{\cirtwo}{\leqslant} \frac{1}{N_{k}^2 n^2 M_k^2} \sum_{i=1}^{N_{k}} |\sset^{(k)}_i| \sum_{r \in \sset^{(k)}_i} \left(\sum_{m \in \cset_k} \left((1+\eta)^\tau - 1\right) |\I_m| \left\| f^{(k)}(\X_m) - \y_m \right\|_1 \right)^2 \\
    & \overset{\cirthree}{\leqslant} \frac{1}{N_{k} n^2 M_k^2} \sum_{i=1}^{N_{k}} |\sset^{(k)}_i| \sum_{r \in \sset^{(k)}_i} \left( \left((1+\eta)^\tau - 1\right) \left\| f^{(k)}(\X^{(k)}) - \y^{(k)} \right\|_2 \right)^2 \\
    & \overset{\cirfour}{\leqslant} \frac{2 R^2}{ \pi \delta^2 \alpha^2 M_k^2} \left(\tau \eta + \frac{\tau(\tau-1)}{2} \eta^2 + o(\eta^2) \right)^2 \normsq{f^{(k)}(\X^{(k)}) - \y^{(k)} }. 
\end{align}
\end{subequations}
where $\cirone$ comes from \eqref{eq:local_diff}, 
$\cirtwo$ holds due to $\|\veca\|_1 \leqslant \|\veca\|_2$, 
$\cirthree$ holds due to $\|\veca\|_1 \leqslant \sqrt{\dim(\veca)}\|\veca\|_2$,  
$\cirfour$ is from \paperlemma{lemma:bound_si}. 
With probability at least $1-\delta$, the inner product term can thus be bounded as
\begin{equation}
    \left \langle \vecd^{\RN{2}}, f^{(k)}(\X^{(k)})  - \mathbf{y} \right \rangle \leqslant  
    \frac{ \sqrt{2} \tau R}{\sqrt{\pi} \delta \alpha M_k} \left(\eta + \frac{(\tau-1)}{2}\eta^2 + o(\eta^2) \right)  \normsq{f^{(k)}(\X^{(k)}) - \y^{(k)} }. \label{eq:bound_inner2_}
\end{equation}
The bound for the difference term is derived as 
\begin{subequations}
\begin{align}
    \normsq{f^{(k+1)}(\X^{(k)}) - f^{(k)}(\X^{(k)})}
    & \leqslant \sum_{i=1}^{N_{k}} \left( \frac{\eta}{\sqrt{n}} \sum_{r=1}^n c_r  \langle \vecv[k+1]_r - \vecv[k]_r  , \mathbf{x}_{i} \rangle  \right)^2 \\
    & \leqslant \frac{1}{M_k^2} \left( \tau\eta + \frac{\tau(\tau-1)}{2} \eta^2 + o(\eta^2) \right)^2 \normsq{f^{(k)}(\X^{(k)}) - \y^{(k)} }. \label{eq:bound_diff_term_}
\end{align}
\end{subequations}
Combine the results of \eqref{eq:bound_inner1_}, \eqref{eq:bound_inner2_} and \eqref{eq:bound_diff_term_}:
\begin{equation}
\begin{aligned}
    \normsq{f^{(k+1)}(\X^{(k)}) - \y^{(k)}} \leqslant 
    &\; \Bigg\{ 1 + \frac{2\eta\tau}{M_k} \Bigg[ \left( \frac{4 \sqrt{2} R}{ \sqrt{\pi} \delta \alpha} + \sqrt{\frac{\ln \left(\frac{2 N_{k}^2}{\delta} \right)}{2 n}} + \frac{\kappa \lambda_{k}}{N_{k}} \right)  \\ 
    & - \frac{(1+\kappa)\lambda_{k}}{N_{k}} \Bigg] 
    +  \frac{\eta^2 \tau^2 }{M_k^2}  + o(\eta^2) \Bigg\} \normsq{f^{(k)}(\X^{(k)}) - \y^{(k)} }. 
\end{aligned}
\end{equation}
Let $R = O\left(\frac{\delta \alpha \lambda_{k}}{N_{k}}\right)$, $n = \Omega\left(\frac{ N_{k}^{2}}{\lambda_{k}^{2}} \ln \frac{N_{k}^{2}}{\delta} \right)$, and $\eta = O\left( \frac{\lambda_{k}}{\tau N_{k} M_k}\right)$, we have 
\begin{subequations}
\begin{align}
    \normsq{f^{(k+1)}(\X^{(k)}) - \y^{(k)}} \leqslant \left( 1 - \frac{\eta \tau \lambda_{k}}{2 N_{k} M_k}  \right) \normsq{f^{(k+1)}(\X^{(k)}) - \y^{(k)} }. 
\end{align}
\end{subequations}
\end{proof}

\end{document}